\documentclass[runningheads]{llncs}
\usepackage{graphicx}
\usepackage{amsmath,amssymb} 
\pdfoutput=1
\usepackage{color}
\usepackage{graphicx}
\usepackage{graphics}
\usepackage{array}
\usepackage{epsfig}
\usepackage{fontenc}
\usepackage{enumitem}
\usepackage{epstopdf}
\usepackage{multirow}
\usepackage{subfig}
\usepackage{algorithm}
\usepackage{algorithmic}
\usepackage{comment}
\usepackage{cite}
\usepackage{tabularx, booktabs}
\usepackage{arydshln}

\newcolumntype{C}[1]{>{\centering\let\newline\\\arraybackslash\hspace{0pt}}m{#1}}
\newcolumntype{L}[1]{>{\raggedright\let\newline\\\arraybackslash\hspace{0pt}}m{#1}}
\newcommand{\blue}[1]{\textcolor{blue}{#1}}

\usepackage{fancyheadings}

\usepackage{array,url,kantlipsum}
\usepackage[hang,flushmargin]{footmisc}
\newcolumntype{Y}{>{\centering\arraybackslash}X}
\begin{document}
\pagestyle{headings}
\mainmatter
\def\ECCVSubNumber{5118}  
%
\title{Unsupervised Learning of Category-Specific Symmetric 3D Keypoints from  Point Sets}

\titlerunning{Unsupervised Learning of Category-Specific Symmetric 3D Keypoints}
%
\author{
Clara Fernandez-Labrador \inst{1,2,3} \and
Ajad Chhatkuli \inst{3} \and
Danda Pani Paudel \inst{3} \and  \\
Jose J. Guerrero \inst{1} \and
Cédric Demonceaux \inst{2} \and
Luc Van Gool \inst{3,4}
}
\authorrunning{Fernandez-Labrador, Chhatkuli, Paudel, Guerrero, Demonceaux, Van Gool}

%
\institute{I3A, University of Zaragoza, Spain  \and
VIBOT ERL CNRS 6000, ImViA, Université de Bourgogne Franche-Comté, France  \and
Computer Vision Lab, ETH Z\"urich, Switzerland  \and
VISICS, ESAT/PSI, KU Leuven, Belgium \\
\email{\{cfernandez,josechu.guerrero\}@unizar.es, cedric.demonceaux@u-bourgogne.fr, \{ajad.chhatkuli,paudel,vangool\}@vision.ee.ethz.ch}}


\maketitle
%
%
%
\begin{abstract}
Automatic discovery of category-specific 3D keypoints from a collection of objects of a category is a challenging problem. The difficulty is added when objects are represented by 3D point clouds, with variations in shape and semantic parts and unknown coordinate frames. We define keypoints to be category-specific, if they meaningfully represent objects’ shape and their correspondences can be simply established order-wise across all objects. This paper aims at learning such 3D keypoints, in an unsupervised manner, using a collection of misaligned 3D point clouds of objects from an unknown category. In order to do so, we model shapes defined by the keypoints, within a category, using the symmetric linear basis shapes without assuming the plane of symmetry  to  be known. The usage of symmetry prior leads us to learn stable keypoints suitable for higher misalignments. To the best of our knowledge, this is the first work on learning such keypoints directly from 3D point clouds for a general category. Using objects from four benchmark datasets, we demonstrate the quality of our learned keypoints by quantitative and qualitative evaluations. Our experiments also show that the keypoints discovered by our method are geometrically and semantically consistent.

\begin{figure}
\centering
\subfloat{\includegraphics[width=1\linewidth]{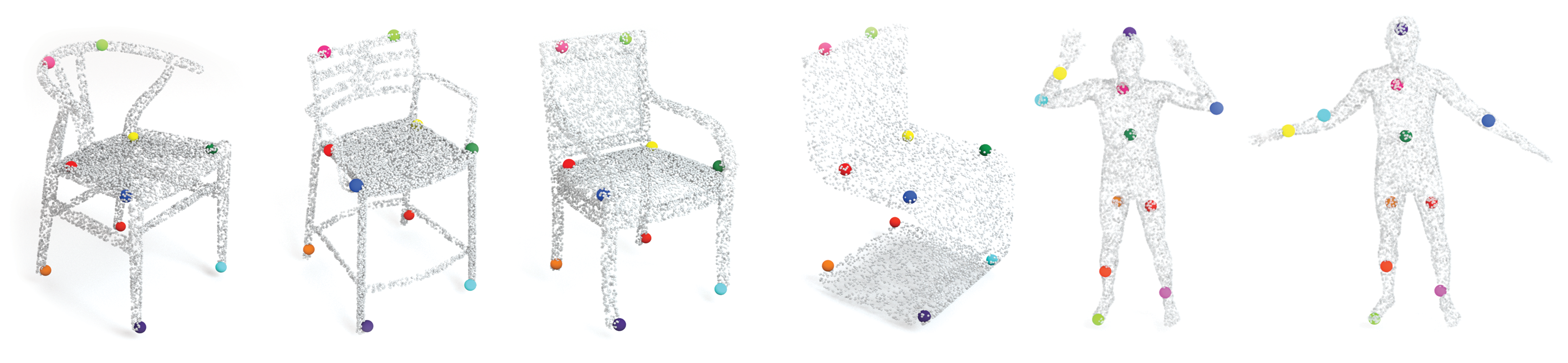}}
\caption{\textbf{Category-specific 3D Keypoints.} The predicted keypoints follow the symmetric linear shape basis prior modeling all instances in a category under a common framework. They not only are consistent across different instances, but also are ordered and correspond to semantically meaningful locations.}
\label{fig:teaser}
\end{figure}
\vspace{-5mm}

\end{abstract}

\section{Introduction}
\label{sec:intro}
A set of keypoints representing any object is historically of large interest for geometric reasoning, due to their simplicity and ease of handling. Keypoints-based methods~\cite{Lowe2004,tola2009daisy,Bay2008} have been crucial to the success of many vision applications. A few examples include; 3D reconstruction \cite{novotny2019c3dpo, Dai2012, Snavely2007}, registration \cite{yew20183dfeat,Kneip2014,Luong1995, Loper2015}, human body pose \cite{shotton2011real,moreno20173d,cao2017,bogo2016smpl}, recognition \cite{he2017mask,sattler2011fast}, and generation \cite{tang2019cycle,zafeiriou20173d}. That being said, many keypoints are defined manually, while considering their semantic locations such as facial landmarks and human body joints, to address the problem at hand. To further benefit from their widespread utility, several attempts have been made on learning to detect keypoints~\cite{huang2017coarse,pavlakos20176, zhang2014facial, dong2018style, yu2016deep}, as well as on automatically discovering them \cite{alahi2012freak, li2019novel, li2019usip,suwajanakorn2018discovery}. In this regard, the task of learning to detect keypoints from several supervision examples, has achieved many successes \cite{wu2016single, pavlakos20176}. However, discovering them automatically from unlabeled 3D data --such that they meaningfully represent shapes and semantics-- so as to have a similar utility as those of manually defined, has received only limited attention due to its difficulty.       

As objects of interest reside in the 3D space, it is not surprising that 3D keypoints are preferred for geometric reasoning. For the given 3D keypoints, their counterparts in 2D images can be associated by merely using camera projection models~\cite{yang2019perfect, hejrati2012analyzing,wang2014robust}. However, being able to directly predict keypoints on provided 3D data (point clouds) has the advantage that the task can be achieved when multiple camera views or images are not available. In this work, we are interested on learning keypoints using only 3D structures. In fact, 3D structures with keypoints suffice for several applications including, registration~\cite{persad2017automatic}, shape completion~\cite{mitra2014structure}, and shape modeling~\cite{reed2013modeling}; without requiring their 2D counterparts.

When 3D objects go through shape variations, due to deformation or when two different objects of a category are compared, consistent keypoints are desired for meaningful geometric reasoning. Recall the examples of semantic keypoints such as facial landmarks and body joints. To serve a similar purpose, \emph{can we automatically find kepoints that are consistent over inter-subject shape variations and intra-subject deformations in a category?} This is the primary question that we are interested to answer in this paper. Furthermore, we wish to discover such keypoints directly from 3D point sets, in an unsupervised manner. We call these keypoints ``category-specific", which are expected to meaningfully represent objects’ shape and offer their correspondence order-wise across all objects.
More formally, we define the desired properties of category-specific keypoints as: i) generalizability over different shape instances and alignments in a category, ii) one-to-one ordered correspondences and semantic consistency, iii) representative of the shape as well as the category while preserving shape symmetry. These properties not only make the representation meaningful, but also tend to enhance the usefulness of keypoints.
Learning category-specific keypoints on point clouds, however, is a challenging problem because not all the object parts are always present in a category. The challenges are exacerbated when the practical cases of misaligned data and unsupervised learning are considered. Related works do not address all these problems, but instead opt for; dropping category-specificity and using aligned data~\cite{li2019usip}, employing manual supervision on 2D images~\cite{pavlakos20176}, or using aligned 3D and multiple 2D images with known pose~\cite{suwajanakorn2018discovery}. The latter method achieves category-specificity without explicitly reasoning on the shapes. Yet another work leverages predefined local shape descriptors and a template model \cite{creusot20123d} specifically on faces.

 
In this paper, we show that the category-specific keypoints with the listed properties can be learned unsupervised by modeling them with non-rigidity, based on unknown linear basis shapes. We further impose an unknown reflective symmetry on the deformation model, when considering categories with instance-wise symmetry. For categories where instance-wise symmetry is not applicable, we propose the use of symmetric linear basis shapes in order to better model, what we define as symmetric deformation spaces, e.g., human body deformations. This allows us to better constrain the pose and the shape coefficients prediction. Our proposed learning method does not assume aligned shapes~\cite{suwajanakorn2018discovery}, pre-computed basis shapes~\cite{pavlakos20176} or known planes of symmetry~\cite{sridhar2019multiview} and all quantities are learned in an end-to-end manner. Our symmetry modeling is powerful and more flexible compared to that of previous NRSfM methods~\cite{gao2016symmetric,sridhar2019multiview}. We achieve this by considering the shape basis for a category and the reflective plane of symmetry as the neural network weight variables, optimized during the training process. The training is done on a single input, circumventing the Siamese-like architecture used in \cite{li2019usip, yew20183dfeat}. At inference time, the network predicts the basis coefficients and the pose in order to estimate the instance-specific keypoints. 
Using multiple categories from four benchmark datasets, we evaluate the quality of our learned keypoints both quantitatively and with qualitative visualization. Our experiments show that the keypoints discovered by our method are geometrically and semantically consistent, which are measured respectively by intra-category registration and semantic part-wise assignments. 
We further show that symmetric basis shapes can be used to model symmetric deformation space of categories such as the human body.
\section{Related Work}
\label{sec:related}
Category-specific keypoints on objects have been extensively used in NRSfM methods, however, only few methods have tackled the problem of estimating them. In terms of the outcome, our work is closest to \cite{suwajanakorn2018discovery}, which learns category-specific 3D keypoints by solving an auxiliary task of rigid registration between multiple renders of the same shape and by considering the category instances to be pre-aligned. Although the method shows promising results on 2D and 3D, it does so without explicitly modeling the shapes. Consequently, it requires renders of different instances to be pre-aligned to reason on keypoint correspondences between instances. A similar task is also solved in \cite{pavlakos20176} for 6-degrees of freedom (DoF) estimation which uses low-rank shape prior to condition keypoints in 3D. Although, the low-rank shape modeling is a powerful tool, \cite{pavlakos20176} requires supervision for heatmap prediction and relies on aligned shapes and pre-computed shape basis. \cite{wu2016single} also predicts keypoints for categories with low-rank shape prior but the method is again trained on fully supervised manner. Moreover, all of the mentioned methods learn keypoints on images as heatmaps and thereafter lift them to 3D. Different from the other works, \cite{creusot20123d} exploits deformation model and symmetry to directly predict keypoints on 3D but requires a face template, aligned shapes and known basis.
Shape modeling of category shape instances has been widely explored in NRSfM works. Linear low-rank shape basis\cite{Bregler2000,Torresani2008,Dai2012}, low-rank trajectory basis~\cite{Akhter2008}, isometry or piece-wise rigidity~\cite{Taylor2010,Parashar2016} are some of the different methods used for NRSfM. Recently, a few number of works have used low-rank shape basis in order to devise learned methods \cite{novotny2019c3dpo,kong2019deep,wu2016single,sridhar2019multiview}. Another useful tool in modeling shape category is the reflective symmetry, which is also directly related to the object pose. Although \cite{gao2016symmetric} showed that the low-rank shape basis can be formulated with unknown reflective symmetry, its adaptation to learned NRSfM methods is not trivial. Recent methods, in fact, assume that the plane of symmetry is one among a few known planes~\cite{wang2019normalized}. Moreover, none of the methods formulate symmetry applicable for non-rigidly deforming objects such as the human body. A parallel work~\cite{wu2020unsupervised} on this regard models symmetry probabilistically in a warped canonical space to reconstruct 3D of different objects.

While shape modeling is a key aspect of our work, another challenge is to infer ordered keypoints by learning on unordered point sets. Despite several advances on deep neural networks for point sets \cite{qi2017pointnet,qi2017pointnet++,verma2018feastnet}, current achievements of learning on images dwarf those of learning on point sets. A related work learns to predict 3D keypoints unsupervised by again solving the auxiliary task of correctly estimating rotations in a Siamese architecture~\cite{bromley1994signature}. The keypoint prediction is done without order by pooling features of certain point neighborhoods. Another previous work~\cite{yew20183dfeat} proposes learning point features for matching, again using alignment as the auxiliary task. Matching such keypoints across shapes is not an easy task as the keypoints are not predicted in any order. In the following sections we show how one can model shape instances using the low-rank symmetric shape basis and use the shape modeling to predict ordered category-specific keypoints.
\section{Background and Theory}
\label{sec:background}
\paragraph{Notations.}
We represent sets and matrices with special Latin characters (e.g., $\mathcal{V}$) or bold Latin characters (e.g., $\mathsf{V}$). Lower or uppercase normal fonts, e.g., $K$ denote scalars. Lowercase bold Latin letters represent vectors as in $\mathsf{v}$. We use lowercase Latin letters to represent indices (e.g., $i$). Uppercase Greek letters represent mappings or functions (e.g., $\Pi$). We use $\mathcal{L}$ to denote loss functions. Finally the operator $\text{mat}(.)$ converts a vector $\mathsf{v}\in \mathbb{R}^{3N\times1}$ to a matrix $\mathsf{M}\in\mathbb{R}^{3\times N}$.

\subsection{Category-specific Shape and Keypoints}
We represent shapes as point clouds, defined as an unordered set of points $\mathsf{S}=\{\mathsf{s}_1, \mathsf{s}_2, \hdots, \mathsf{s}_M\},\ \mathsf{s}_{j}\in\mathbb{R}^3$, $j\in \{1, 2, \dots, M\}$. The set of all such shapes in a category defines the category shape space $\mathcal{C}$. We write a particular $i$-th category-specific shape instance in $\mathcal{C}$ as $\mathsf{S}_i$. For convenience, we will use the terms category-specific shape and shape interchangeably. The category shape space $\mathcal{C}$ can be anything from a set of discrete shapes to a smooth manifold of category-specific shapes spanned by a deformation function $\Psi_\mathcal{C}$. The focus of the work is on learning meaningful 3D keypoints from the point set representation of $\mathsf{S}_i$. To that end, this section defines category-specific keypoints and develops their modeling.

\paragraph{Category-specific keypoints.} We represent category-specific keypoints of a shape $\mathsf{S}_i$ as a sparse tuple of points, $\mathsf{P}_i=(\mathsf{p}_{i1}, \mathsf{p}_{i2}, \hdots, \mathsf{p}_{iN}),\ \mathsf{p}_{ij}\in\mathbb{R}^3$,  $j\in \{1, 2, \dots, N\}$. Unlike the shape, its keypoints are represented as ordered points. Our objective is to learn a mapping $\Pi_\mathcal{C}: \mathsf{S}_i \to \mathsf{P}_i$ in order to obtain the category-specific keypoints from an input shape $\mathsf{S}_i$ in $\mathcal{C}$.
Although not completely unambiguous, we can define the category-specific keypoints using the properties listed in Sec.~\ref{sec:intro}. In mathematical notations they are:
\begin{enumerate}[label=(\roman*)]
    \item Generalization: $\Pi_\mathcal{C}(\mathsf{S}_i) = \mathsf{P}_i, \ \forall \mathsf{S}_i\in \mathcal{C}$. 
    \item Corresponding points and semantic consistency: Given $\mathsf{S}_a, \mathsf{S}_b \in \mathcal{C}$, we want $\mathsf{p}_{aj} \Leftrightarrow \mathsf{p}_{bj}$. Similarly, $\mathsf{p}_{aj}$ and $\mathsf{p}_{bj}$ should have the same semantics.
    \item Representative-ness: $\text{vol}(\mathsf{S}_i) = \text{vol}(\mathsf{P}_i)$ and $\mathsf{p}_{ij} \in \mathsf{S}_i$, where $\text{vol(.)}$ is the Volume operator for a shape. If $\mathsf{S}_i\in \mathcal{C}$ has a reflective symmetry, $\mathsf{P}_i$ should have the same symmetry.
\end{enumerate}

\subsection{Category-specific Shapes as Instances of Non-Rigidity}
Several recent works have modeled shapes in a category as instances of non-rigid deformations~\cite{novotny2019c3dpo,kong2019deep,wu2016single,sridhar2019multiview}. The motivation lies in the fact that such shapes often share geometric similarities. Consequently, there likely exists a deformation function $\Psi_\mathcal{C}: \mathsf{S}_T \to \mathsf{S}_i$, which can map a global shape property $\mathsf{S}_T$ (shape template or basis shapes) to a category shape instance $\mathsf{S}_i$. However, we argue that modeling $\Psi_\mathcal{C}$ is not trivial and in fact a convenient representation of $\Psi_\mathcal{C}$ may not exist in many cases. This observation, in fact, is what makes the dense Non-Rigid Structure-from-Motion (NRSfM) so challenging. On the other hand, one can imagine a deformation function $\Phi_\mathcal{C}: \mathsf{P}_T \to \mathsf{P}_i$, going from a global keypoints property $\mathsf{P}_T$ to the category-specific keypoints $\mathsf{P}_i$. The deformation function $\Phi_\mathcal{C}$ thus satisfies: $\mathsf{p}_{ij}\in \Phi_\mathcal{C}$ implies $\mathsf{p}_{ij} \in \Psi_\mathcal{C}$ and effectively, $\Phi_\mathcal{C} \subset \Psi_\mathcal{C}$, if the set order in $\mathsf{P}_i$ is ignored. Unlike $\Psi_\mathcal{C}$, the deformation function $\Phi_\mathcal{C}$ may be simple enough to model and use for estimating the category-specific keypoints $\mathsf{P}_i$. We therefore, choose to seek the non-rigidity modeling in the space of keypoints $\mathcal{P}=\{\mathsf{P}_1, \mathsf{P}_2, \dots, \mathsf{P}_L\}$, which functions as an abstraction of the space $\mathcal{C}$. Non-rigidity can be used to define the prediction function $\Pi_\mathcal{C}$ as below:
\begin{equation}
\label{eq:Phi}
\Pi_\mathcal{C}(\mathsf{S}_i;\theta) = \Phi_\mathcal{C}(\mathsf{r}_i; \theta)= \mathsf{P}_i
\end{equation}
where $\theta$ denotes the constant function parameters of $\Pi_\mathcal{C}$ and $\mathsf{r}_i$ is the predicted instance specific vector parameter. In our problem, we want to learn $\theta$ from the example shapes in $\mathcal{C}$ without using the ground-truth labels, supervised by $\Phi_\mathcal{C}$. In the NRSfM literature, two common approaches of modeling shape deformations are the low-rank shape prior~\cite{Bregler2000,Torresani2008,Dai2012,Akhter2008} and the isometric prior~\cite{Taylor2010,Parashar2016}. In this paper, we investigate the modeling using the low-rank shape prior, with instance-wise symmetry as well as symmetry of the deformation space.

\subsection{Low-Rank Non-rigid Representation of Keypoints}
The NRSfM approach of low-rank shape basis comes as a natural extension of the rigid orthographic factorization prior~\cite{tomasi1992shape} and was introduced by Bregler et al.~\cite{Bregler2000}. The key idea is that a large number of object deformations can be explained by linearly combining a smaller $K$ number of basis shapes at some pose. In the rigid case, this number is one, hence the rank is 3. In the non-rigid case, it can be higher, while the exact value depends on the complexity of the deformations. Consider $F$ shape instances in $\mathcal{C}$ and $N$ points in each keypoints instance $\mathsf{P}_i$. The following equation describes the projection with shape basis.
\begin{equation}
\label{eq:shapebasis}
\mathsf{P}_i = \Phi_\mathcal{C}(\mathsf{r}_i;\theta)= \mathsf{R}_i\, \text{mat}(\mathcal{B}_\mathcal{C}\, \mathsf{c}_i)\,
\end{equation}
where $\mathcal{B}_\mathcal{C}=(\mathsf{B}_1, \dots, \mathsf{B}_K), \mathcal{B}_\mathcal{C} \in \mathbb{R}^{3N\times K}$ forms the low-rank shape basis. The rank is lower than the maximum possible rank of $3F$ or $N$ for $3K < 3F$ or $3K < N$. The vector $\mathsf{c}_i\in\mathbb{R}^{K}$ denotes the coefficients that linearly combines different basis for the keypoints instance $i$. Each keypoints instance is then completely parametrized by the basis $\mathcal{B}_\mathcal{C}$ and the coefficients $\mathsf{c}_i$. Next, the projection matrix $\mathsf{R}_i\in SO_3$ is simply the rotation matrix for the shape instance $i$.


Unlike in NRSfM, the problem of computing the category-specific keypoints, has $\mathsf{P}_i$ as unknown. Similar to NRSfM, the rest of the quantities in Eq.~\eqref{eq:shapebasis} -- $\mathsf{c}_i$, $\mathcal{B}_\mathcal{C}$ and $\mathsf{R}_i$ are also unknown. This fact makes our problem doubly hard. First the problem becomes more than just lifting the 2D keypoints to 3D and second, the order of keypoints present in the NRSfM measurements matrix is not available. We intend to solve the aforementioned problems by learning based on Eq.~\eqref{eq:shapebasis}, which is related to the deformation representation of $\Phi_\mathcal{C}$ in Eq.~\eqref{eq:Phi}. Here, $\theta$ includes the global parameters or basis $\mathcal{B}_\mathcal{C}$ and $\mathsf{r}_i$ includes the instance-wise pose $\mathsf{R}_i$ and coefficients $\mathsf{c}_i$. To further reduce ambiguities on pose, we propose to also compute the reflective plane of symmetry for a category.

\subsection{Modeling Symmetry with Non-Rigidity}
\label{sec:sym}
Many object categories have shapes which exhibit a fixed reflective symmetry over the whole category. To discover and use symmetry, we consider two different priors: instance-wise symmetry and symmetric deformation space.
\paragraph{Instance-wise symmetry.}
Instance-wise reflective symmetry about a fixed plane is observed in a large number of rigid object categories (e.g. ShapeNet~\cite{yi2016scalable} and ModelNet~\cite{wu20153d}). Such a symmetry has been previously combined with the shape basis prior in NRSfM~\cite{gao2016symmetric}, however, a convenient representation for learning both the symmetry and the shapes have not been explored yet. A recent learning-based method~\cite{wang2019normalized ,sridhar2019multiview} uses the symmetry prior by performing an exhaustive search over a few planes in order to predict symmetric dense non-rigid shapes. However, such a strategy may not work when the shapes are not perfectly aligned. Instance-wise symmetry can be included by re-writing Eq.~\eqref{eq:shapebasis} as follows:
\begin{align}
    \label{eq:symmetry}
    \begin{split}
    \mathsf{P}_{i\frac{1}{2}} =
    \mathsf{R}_i\, \text{mat}(\mathcal{B}_{\mathcal{C}\frac{1}{2}}\, \mathsf{c}_i), \quad
    \mathsf{P}_i = \begin{bmatrix} \mathsf{P}_{i\frac{1}{2}} & A_\mathcal{C} \mathsf{P}_{i\frac{1}{2}} \end{bmatrix}
    \end{split}
\end{align}
where $\mathsf{P}_{i\frac{1}{2}} \in \mathbb{R}^{3\times N/2}$ represents one half of the category-specific keypoints. $\mathsf{P}_{i\frac{1}{2}}$ is reflected using $\mathsf{A}_\mathcal{C} \in \mathbb{R}^{3\times 3}$ and concatenated to obtain the final keypoints. Due to the exact instance-wise symmetry, we similarly can parametrize the basis as $\mathcal{B}_{\mathcal{C}\frac{1}{2}}\in\mathbb{R}^{3N/2\times K}$ to denote the shape basis for the first half of the keypoints. The reflection operator $\mathsf{A}_\mathcal{C}$ is parametrized by a unit normal vector $\mathsf{n}_\mathcal{C}\in \mathbb{R}^3$ of the plane of symmetry passing through the origin. The advantage of going from Eq.~\eqref{eq:shapebasis} to Eq.~\eqref{eq:symmetry} should be apparent from the reduced dimensionality of the unknowns in $\mathcal{B}_\mathcal{C}$ as well as the additional second equality constraint of Eq.~\eqref{eq:symmetry}, which reduces the ambiguities in NRSfM~\cite{gao2016symmetric}.
\paragraph{Symmetric deformation space.}
In many non-rigid objects, shape instances are not symmetric. However, symmetry may still exist in the deformation space, e.g., in a human body. Suppose that a particular shape instance $\mathsf{S}_k\in \mathcal{C}$ has the reflective symmetry about $\mathsf{n}_\mathcal{C}$, which allows us to define its two halves: $\mathsf{S}_{k\frac{1}{2}}$ and $\mathsf{S}'_{k\frac{1}{2}}$ and thus correspondingly for all shape instances.
\begin{definition}[Symmetric deformation space]
$\mathcal{C}$ is a symmetric deformation space if for every half shape deformation instance $\mathsf{S}_{i\frac{1}{2}}$, there exists any shape instance $\mathsf{S}_j\in \mathcal{C}$ such that the $\mathsf{S}'_{j\frac{1}{2}}$ is symmetric to $\mathsf{S}_{i\frac{1}{2}}$.
\end{definition}
The above definition also applies for the keypoints shape space $\mathcal{P}$.
The instance-wise symmetric space is a particular case of the above. However, Eq.~\eqref{eq:symmetry} cannot model the keypoints instances in the symmetric deformation space. We model such keypoints by introducing symmetric basis that can be weighted asymmetrically, thereby, obtaining the following:
\begin{align}
    \label{eq:basissymmetry}
    \begin{split}
    &\mathsf{P}_i = \mathsf{R}_i
    \begin{bmatrix} \text{mat}(\mathcal{B}_{\mathcal{C}\frac{1}{2}}\, \mathsf{c}_i) & \text{mat}(\mathcal{B}'_{\mathcal{C}\frac{1}{2}}\, \mathsf{c}'_i)\end{bmatrix} 
    \end{split}
\end{align}


where $\mathcal{B}'_{\mathcal{C}\frac{1}{2}}$ is obtained by reflecting $\mathcal{B}_{\mathcal{C}\frac{1}{2}}$ with $A_\mathcal{C}$ and $\mathsf{c}'_i \in \mathbb{R}^K$ forms the coefficients for the second half of the basis. Although Eq.~\eqref{eq:basissymmetry} increases the dimension of the unknowns in the coefficients over Eq.~\eqref{eq:shapebasis}, the added modeling of the symmetry of the deformation space and the reduced dimensionality of the basis can improve the final keypoints estimate. This brings us to the following proposition.

\begin{proposition}
\label{prop:symdef}
Provided that $\mathcal{B}_{\mathcal{C}\frac{1}{2}}$ and $\mathcal{B}'_{\mathcal{C}\frac{1}{2}}$ are symmetric about a plane, Eq.~\eqref{eq:basissymmetry} models a symmetric deformation space if the estimates of $\mathsf{c}_i$ and $\mathsf{c}_i'$ come from the same probabilistic distribution.
\end{proposition}
\begin{proof}
The proof is straightforward and provided in the supplementary material.
\end{proof}
As a consequence of Proposition \ref{prop:symdef}, we can model keypoints in non-rigid symmetric objects with Eq.~\eqref{eq:basissymmetry}, while also tightly modeling the symmetry as long as we maintain the distribution of $\mathsf{c}$ and $\mathsf{c}'$ to be the same.

\section{Learning Category-specific Keypoints}
\label{sec:learning}
In this section, we use the modeling of $\Phi_\mathcal{C}$ to describe the unsupervised learning process of the category-specific keypoints. More precisely, we want to learn the function $\Pi_\mathcal{C}: \mathsf{S}_i \to \mathsf{P}_i$ as a neural network of parameters $\theta$, using the supervisory signal from $\Phi_\mathcal{C}$. In regard to learning keypoints on point sets, recent work \cite{li2019usip} trains a Siamese network to predict order-agnostic keypoints stable to rotations for rigid objects~\cite{li2019usip}. Part of our network architecture is inspired from \cite{li2019usip}, which is based on PointNet~\cite{qi2017pointnet}. However, we use a single input avoiding the expensive Siamese training. The network architecture is shown in Fig.~\ref{fig:overview}, whose input consists of a single shape $\mathsf{S}_i$ misaligned in $SO_{2}$. This is reasonable since point clouds are usually aligned to the vertical direction. 
We describe the different components of the network architecture below.

\begin{figure}
\centering
\subfloat{\includegraphics[width=0.95\linewidth]{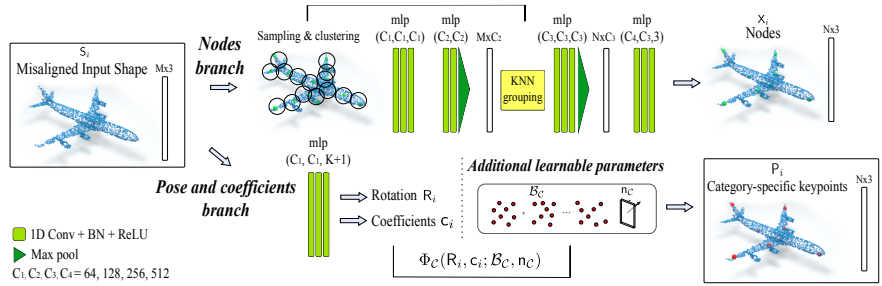}}
\caption{\textbf{Network architecture:} The \textit{pose and coefficients branch} and the \textit{additional learnable parameters} generate the output category-specific keypoints. The \textit{nodes branch} estimates the nodes that guide the learning process. “mlp” stands for multi-layer perceptron. Refer to Sec.~\ref{sec:background} for the modeling, Sec.~\ref{sec:learning} for learning.}
\label{fig:overview}
\end{figure}


\paragraph{Node branch.} This branch estimates a sparse tuple of nodes that are potentially category-specific keypoints but are not ordered. We denote them as $\mathsf{X}_i = \{\mathsf{x}_{i1}, \mathsf{x}_{i2}, \dots, \mathsf{x}_{iN}\}, \mathsf{x}_{ij}\in\mathbb{R}^3$ and $j\in \{1, 2, \dots, N\}$. Initially, a predefined number of nodes $N$ are sampled from the input shape using the Farthest Point Sampling (FPS) and a local neighborhood of points is built for each node with point-to-node grouping \cite{li2018so,li2019usip}, creating $N$ clusters which are mean normalized inside the network. Every point in $\mathsf{S}_{i}$ is associated with one of these nodes. The branch consists of two PointNet-like \cite{qi2017pointnet} networks followed by a kNN grouping layer that uses the initial sampled nodes to achieve hierarchical information aggregation. Finally, the local feature vectors are fed into a Multi-Layer Perceptron (MLP) that outputs the nodes.

\paragraph{Pose and coefficients branch.} We predict the quantities $\mathsf{R}_i$ and $\mathsf{c}_i$ with this branch. We use a single rotation angle to parametrize $\mathsf{R}_i$. The branch consists of an MLP that estimates the mentioned parameters. The output size varies depending on whether we are interested in symmetric shape instances as in Eq.~\eqref{eq:symmetry} or symmetric basis as in Eq.~\eqref{eq:basissymmetry}, the size being double in the latter.

\paragraph{Additional learnable parameters.} Several unknown quantities in Eq.~\eqref{eq:symmetry} or \eqref{eq:basissymmetry} are constant for a category shape space $\mathcal{C}$. Such quantities need not be predicted instance-wise. We rather choose to optimize them as part of the network parameters $\theta$. They are the shape basis $\mathcal{B}_{\mathcal{C}} \in \mathbb{R}^{3N\times K}$ and the unit normal of the plane of symmetry $\mathsf{n}_\mathcal{C}\in\mathbb{R}^3$. We observed that a good choice for the number of shape basis is $5\leq K\leq10$. In fact, the generated keypoints are not very sensitive to the choice of $K$, as a large $K$ tends to generate sparser shape coefficients and similar keypoints. Depending upon the problem, alternate parametrization can be considered for $\mathsf{n}_\mathcal{C}$, e.g., Euler angles.
\\

At inference time, we apply Non-Maximal Suppression obtaining the final $N'$ number of keypoints. Our method consistently provides $N'$ keypoints for all instances in the category, as they follow the same geometric model.

\subsection{Training Losses}
In order to adhere to the definitions of the category-specific keypoints introduced in Sec.~\ref{sec:intro} as well as our shape modeling, we design our loss functions as below.

\paragraph{Chamfer loss with symmetry and non-Rigidity.}
Eq.~\eqref{eq:Phi} suggests that the neural network $\Pi_\mathcal{C}$ can be trained with an $\ell_2$ loss between the node predictions $\mathsf{X}_i$ and the deformation function $\mathsf{P}_i= \Phi_\mathcal{C}(\mathsf{R}_i, \mathsf{c}_i; \mathcal{B}_\mathcal{C}, \mathsf{n}_\mathcal{C})$, thus obtaining $\mathsf{P}_i= \mathsf{X}_i$. However, as confirmed by our evaluations as well as in \cite{li2019usip}, the $\ell_2$ loss does not converge as the network is unable to predict the point order. Alternatively, the Chamfer loss \cite{fan2017point} does converge, minimizing the distance between each point $\mathsf{x}_{ik}$ in the first set $\mathsf{X}_{i}$ and its nearest neighbor $\mathsf{p}_{ij}$ in the second set $\mathsf{P}_{i}$ and vice versa.

\begin{equation}
\label{eq:chamfer}
     \mathcal{L}_{chf} = \sum_{k=1}^{N}\underset{\mathsf{p}_{ij} \in \mathsf{P}_i}{\text{min}}\|\mathsf{x}_{ik} - \mathsf{p}_{ij}\|^2_2 + \sum_{j=1}^{N}\underset{\mathsf{x}_{ik} \in \mathsf{X}_i}{\text{min}}\|\mathsf{x}_{ik} - \mathsf{p}_{ij}\|^2_2,
\end{equation}

The Chamfer loss in Eq.~\eqref{eq:chamfer} ensures that the learned keypoints follow a generalizable category-specific property -- that they are a linear combination of common basis learned specifically for the category. To additionally model symmetry, Eq.~\eqref{eq:symmetry} or \eqref{eq:basissymmetry} is directly used in Eq.~\eqref{eq:chamfer}. Therefore, two different Chamfer losses are possible modeling two different types of symmetries.




\paragraph{Coverage and inclusivity loss.}
The Chamfer loss in Eq.~\eqref{eq:chamfer} does not ensure that the keypoints follow the object shape. However, one can add the following conditions: a) the keypoints cover the whole category shape (coverage loss), b) the keypoints are not far from the point cloud (inclusivity loss). 
The coverage loss can be defined as a Huber loss between the volume of the nodes $\mathsf{X}_{i}$ and that of the input shape $\mathsf{S}_{i}$, using the product of the singular values. However, we instead approximate the volume using the 3D bounding box defined by the points. This improves the training speed and, based on our initial evaluations, also does not harm performance. The coverage loss is thus given by:

\begin{equation}
    \label{eq:cov}
    \mathcal{L}_{cov} = \|\text{vol}(\mathsf{X}_{i}) - \text{vol}(\mathsf{S}_{i})\|
\end{equation}

The inclusivity loss is formulated as a single side Chamfer loss \cite{besl1992method} which penalizes nodes in $\mathsf{X}_{i}$ that are far from the original shape $\mathsf{S}_{i}$, similarly to Eq.~\eqref{eq:chamfer}: 

\begin{equation} 
\label{eq:inc}
    \mathcal{L}_{inc} = \sum_{k=1}^{N}\underset{\mathsf{s}_{ij} \in \mathsf{S}_{i}}{\text{min}}\|\mathsf{x}_{ik} - \mathsf{s}_{ij}\|_2^2.
\end{equation}
%
\section{Experimental Results}
We conduct experiments to evaluate the desired properties of the proposed category-specific keypoints and show their generalization over indoor/outdoor objects and rigid/non-rigid objects with four different datasets in total (Sec.~\ref{subsec:generalization},~\ref{subsec:sem}). All these properties are also compared with a proposed baseline. We then evaluate the practical use of our keypoints for intra-category shapes registration (Sec.~\ref{subsec:registration}), analyzing the influence of symmetry. Additional qualitative results are shown in Fig.~\ref{fig:teaser} and the supplementary material.

\paragraph{Datasets.} We use four main datasets. They are ModelNet10~\cite{wu20153d}, ShapeNet parts~\cite{yi2016scalable}, Dynamic FAUST~\cite{bogo2017dynamic} and Basel Face Model 2017~\cite{gerig2018morphable}. Since our method is category-specific, we require separate training data for each class in the datasets. 
For indoor rigid objects, we choose three categories from ModelNet10~\cite{wu20153d}; chair, table and bed. 
Three outdoor rigid object categories: airplane, car and motorbike, are evaluated from ShapeNet parts~\cite{yi2016scalable}. 
For non-rigid objects, we randomly choose a sequence of the Dynamic Faust~\cite{bogo2017dynamic}, that provides high-resolution 4D scans of human subjects in motion. Finally, we generate shape models of faces using the Basel Face Model 2017~\cite{gerig2018morphable} combining 50 different shapes and 20 different expressions.
All models are normalized in the range $-1$ to $1$ and are randomly misaligned within $\pm 45$ degrees. 

\paragraph{Baseline.} Since this is the first work computing category-specific keypoints from point sets, we construct our own baseline based on the recent work USIP~\cite{li2019usip}. The method detects stable interest points in 3D point clouds under arbitrary transformations and is also unsupervised, which makes it the closest method for comparison. The USIP detector is not category-based, so we train the network per category to create the baseline. Additionally, we adapt the number of predicted keypoints so that the results are directly comparable to ours. While training with some of the categories, specifically car and bed, we observe that predicting lower number of keypoints can lead to some degeneracies~\cite{li2019usip}.

\paragraph{Implementation details.}
Input point clouds of dimension $3\times2000$ are used. We implement the network in Pytorch \cite{NEURIPS2019_9015} and train it end-to-end from scratch using the Adam optimizer \cite{kingma2014adam}. The initial learning rate is $10^{-3}$, which is exponentially decayed by a rate of $0.5$ every $40$ epochs. We use a batch size of $32$ and train each model until convergence, for $200$ epochs. The final loss function combines the three training losses, Eqs. \eqref{eq:chamfer}, \eqref{eq:cov} and \eqref{eq:inc}, and are weighted as follows: $w_{chf} = w_{cov} = 1$ and $w_{inc} = 2$.
For ModelNet10 and ShapeNet parts, we use the training and testing split provided by the authors. For the Basel Face Model 2017, we follow the common practice and split the 1000 generated faces in 85\% training and 15\% test. We use the same split strategy for the sequence `50009\_jiggle\_on\_toes' of Dynamic Fuaust, which contains 244 examples.

\begin{table}[h]
\small
\centering
\resizebox{0.83\textwidth}{!}{
\begin{tabular}{@{\extracolsep{0.1pt}}c cccccc}
\textit{Category} & \textit{Coverage} & \textit{Model Err} & \textit{Correspondence} & \textit{Inclusivity} & \textit{Sym Err} & \textit{Definition}\\
\hline
& $\%$ & $\%$ & $\%$ & $\%$ & $^{\circ}$ & \\
chair & $\textbf{88.83}$ & $0.72$ & $\textbf{100}$ & $90.46$ & $0.40$ & $10$      \\
table & $\textbf{93.33}$ & $0.99$  & $\textbf{100}$ & $93.38$ & $2.86$ & $6$     \\
bed & $\textbf{80.31}$ & $0.94$ & $\textbf{100}$ & $\textbf{95.33}$ & $0.13$ & $6$   \\ 

airplane & $\textbf{89.15}$ & $0.64$ & $\textbf{100}$ & $\textbf{96.35}$ & $0.20$ &  $8$    \\
car & $\textbf{92.39}$ & $0.72$ & $\textbf{100}$ & $\textbf{97.77}$ & $2.21$ & $8$     \\
motorbike & $\textbf{96.13}$ & $0.79$ & $\textbf{100}$ & $\textbf{90.53}$ & $1.42$ & $8$   \\

human body & $\textbf{85.59}$ & $0.72$ & $\textbf{100}$ & $97.73$ & $33.30$ & $11$  \\

faces & $\textbf{97.93}$ & $0.41$ & $\textbf{100}$ & $\textbf{100}$  & $0.15$ & $9$ \\
\noalign{\smallskip}
\hline
\noalign{\smallskip}
chair & $79.73$ & $-$ & $55.6$ & $\textbf{98.50}$ & $-$ & $10$  \\
table & $79.72$ & $-$ & $34.5$ & $\textbf{99.83}$ & $-$ & $6$   \\
bed   & $42.18$ & $-$ & $49.33$ & $70.00$ & $-$ & $6$   \\ 

airplane  & $69.24$ & $-$ & $47.5$ & $87.13$ & $-$ &  $8$    \\
car       & $26.87$ & $-$ & $32.18$ & $74.0$  & $-$ & $8$     \\
motorbike & $75.29$ & $-$ & $48.14$ & $84.57$ & $-$ & $8$   \\

human body & $72.66$ & $-$ & $50.45$ & $\textbf{100}$ & $-$ & $11$    \\

faces & $42.98$ & $-$ & $30.11$ & $\textbf{100}$ & $-$ & $9$ \\
\noalign{\smallskip}
\hline
\noalign{\smallskip}
\end{tabular}
}
\caption{\textit{Properties Analysis:} Top (ours) and bottom (baseline~\cite{li2019usip}). For coverage, correspondence and inclusivity \textit{higher is better}, and for model and symmetry error \textit{lower is better}. We empirically show the desired properties of our keypoints, as well as the generalization of our method over indoor/outdoor and rigid/non-rigid objects. Best results are in bold.}
\label{tab:eval}
\end{table}

\begin{figure}
\centering
\subfloat{\includegraphics[width=0.9\linewidth]{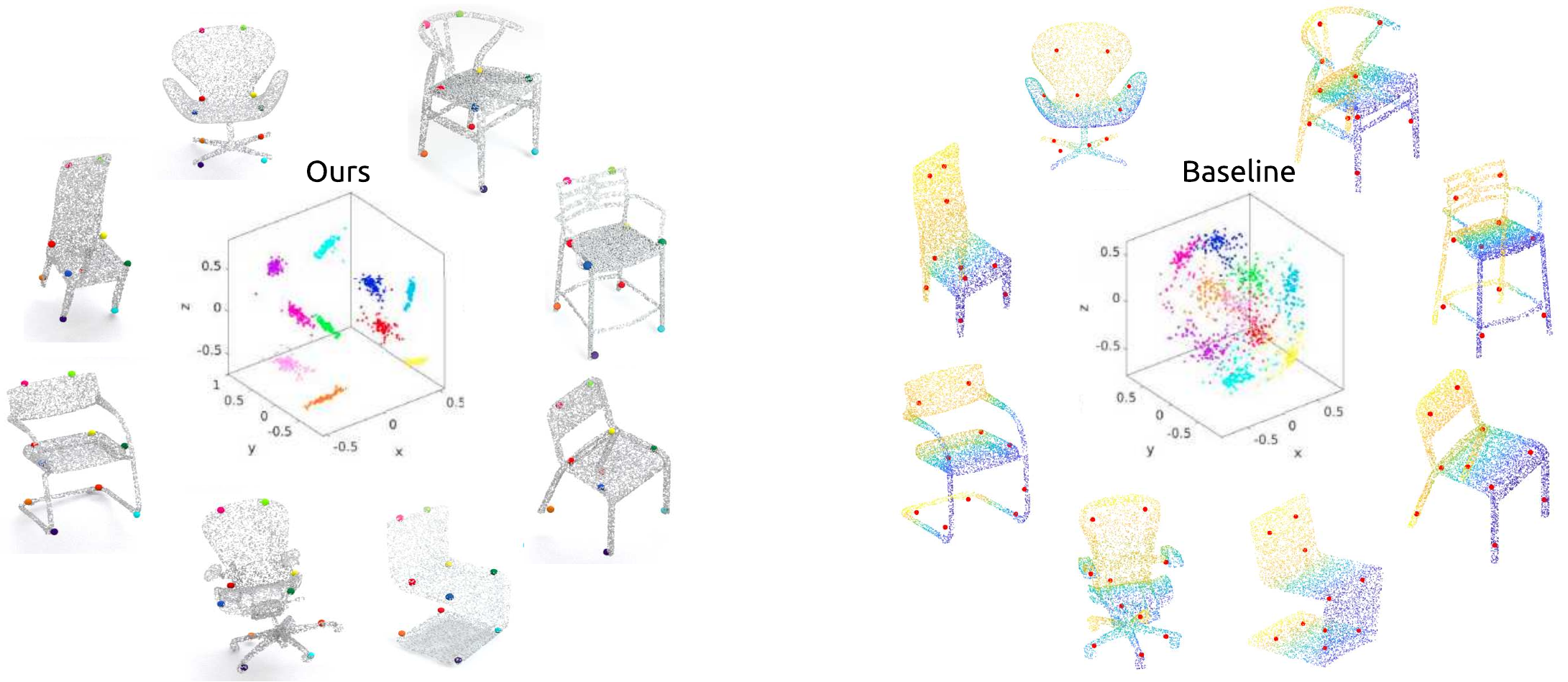}} 
\caption{\textbf{Keypoints correspondence/repeatability across instances}. We cluster the predicted keypoints for all the instances in the category to show their geometric consistency. Note how our keypoints are neatly clustered as they are consistently predicted in the corresponding geometric locations, unlike the baseline keypoints. (Note: cluster colors do not correspond to keypoint colors.)}
\label{fig:kmeans}
\end{figure}

\subsection{Desired Properties Analysis}
\label{subsec:generalization}
As described in Sec.~\ref{sec:intro} and~\ref{sec:background}, the category-specific keypoints satisfy certain desired properties. We propose six different metrics to evaluate the properties which are also used for comparison against the baseline. All the results are presented in Table~\ref{tab:eval}, and are averaged across the test samples.

\paragraph{Coverage}: According to property \textit{iii)}, we seek keypoints that are representative of each instance shape as well as of the category itself. To measure it, we calculate the percentage of the input shape covered by the keypoints' 3D bounding box. On average, we achieve a $29.4\%$ more coverage than the baseline. 

\paragraph{Model Error}: This metric refers to the Chamfer distance between the estimated nodes and the learned category-specific keypoints, normalized by the model's scale. We obtain less than $1\%$ of error in all the categories, meaning that the network satisfactorily manages to generalize, describing the nodes with the symmteric non-rigidity modeling (Properties \textit{i)} and \textit{iii)}).

\paragraph{Correspondence/ Repeatability}: We measure the ability of the model to find the same set of keypoints on different instances of a given category (Property \textit{ii)}). For our method, we cluster the keypoints using their inherent order whereas for the baseline, we use K-means clustering to evaluate and compare this property. We show the evaluation in Fig.~\ref{fig:kmeans}, the rest of the categories are provided in the supplementary material. One can see at a glance how our keypoints are well clustered, unlike the baseline keypoints. Numerically, we show the \% occurrence of each specific keypoint belonging to the same cluster across instances. Our keypoints satisfy $100\%$ the correspondence/repeatability test thanks to our geometric non-rigidity modelling.

\paragraph{Inclusivity}: We measure the percentage of keypoints that lie inside the point cloud (of scale 2) within a chosen threshold of 0.15, which also proves property \textit{iii)}. This is the only metric in which our method doesn't outperform the baseline in all cases. On average, our method achieves $\sim95\%$ inclusivity compared to $\sim89\%$ for the baseline.

\paragraph{Symmetry:} The metric shows the angle error of the predicted reflective plane of symmetry. We obtain highly accurate prediction for rigid categories. In the non-rigid human body shape however, the ambiguities are severe. Despite that, the learned keypoints satisfy the other properties, particularly that of semantic correspondence. Both of these facts can be observed in Fig.~\ref{fig:teaser}. 

\paragraph{Definition:} final number of keypoints $N'$ predicted per category after the Non-Maximal Suppression. 

\subsection{Semantic Consistency}
\label{subsec:sem}
We use the ShapeNet part dataset \cite{yi2016scalable} to show the semantic consistency of the proposed keypoints. Following the low-rank non-rigidity modelling, the keypoints lie on geometrically corresponding locations. The idea of the experiment is to measure keypoint-semantics relationship for every keypoint across instances of the category. The results are presented in Fig.~\ref{fig:semCovMatOURS} as covariance matrices, along with keypoint visualizations per category for our method. On average, the proposed keypoints have a high semantic consistency of $93\%$ across instances, despite the large intra-category variability. 
The same experiment is performed for the baseline and presented in bottom of Fig.~\ref{fig:semCovMatOURS}. Here, the degeneracy causes all the keypoints to approach the object centroid for `Car'. Nonetheless, we observe no semantic consistency even for `Airplane' without degeneracies.
Our model, aiming for a common representation for all the instances of the category, avoids placing keypoints in less representative parts or unique parts, e.g., arm rests in chairs (in Fig.~\ref{fig:teaser}), engines in airplanes or gas tank in motorbikes. This highlights significant robustness achieved in modelling and learning the keypoints.

\begin{figure}
\centering
\includegraphics[width=0.83\linewidth]{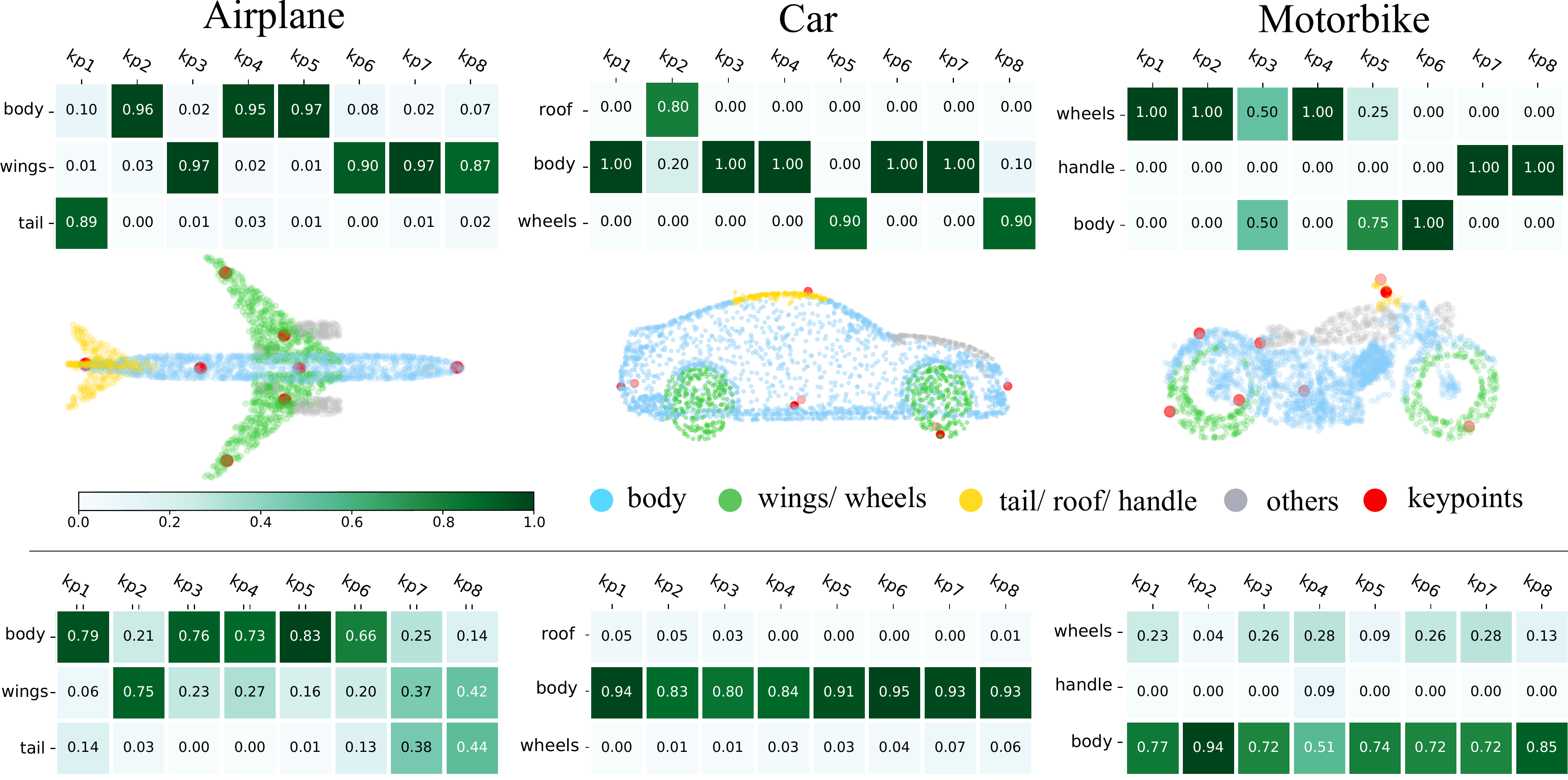}
\caption{\textbf{Semantic part correspondence.} Top to bottom: the semantic correspondence for the proposed keypoints, qualitative results and the baseline semantic correspondence. Our predicted keypoints show the correct semantic correspondence across the category.}
\label{fig:semCovMatOURS}
\end{figure}

\subsection{Object Pose and Intra-category Registration}
\label{subsec:registration}
Previous methods do not handle misaligned data due to the obvious difficulty it poses to unsupervised learning. This deserves special attention since real data is never aligned.
In this section we evaluate the intra-category registration performance of our model and show the impact of the different symmetry models proposed. These results implicitly measure the object poses estimated as well.
\paragraph{Rotation Ambiguities.} Recent unsupervised approaches for keypoint detection actually self-supervise rotation during training, e.g.,~\cite{suwajanakorn2018discovery,li2019usip}, and highlight that it is crucial for achieving a good performance. 
In our case, we do not directly supervise the rotations. Therefore, the different combination of basis shapes can result in different alignments. This implies that computing $\mathsf{P}_i$ with the deformation function $\Phi_\mathcal{C}$ will give the correct set of keypoints along with the correct plane of symmetry, but the predicted rotation alone is not meaningful for registration. 
\paragraph{Experimental setup.} Despite the above ambiguity, an important characteristic of the proposed keypoints is that they are ordered, which empowers direct inter-instances registration since no extra descriptors are needed for matching. We perform experiments for the chair category, using 10 keypoints (Table~\ref{tab:eval}) and a misalignment of $\pm45$ degrees.
Three different models are compared. The first one is trained without symmetry awareness following Eq.~\eqref{eq:shapebasis}. A second one uses shape symmetry during training as shown in Eq.~\eqref{eq:symmetry}. The last model is trained with basis symmetry as in Eq.~\eqref{eq:basissymmetry}.
We attempt to register keypoints in each instance to those of randomly chosen three aligned templates by computing a similarity transformation and observe the mean error.
Fig.~\ref{fig:rot} shows that symmetry helps to have more control over the rotations and tackle higher misalignment. More results and analysis are provided in the supplementary.

\begin{figure}
\centering
\subfloat{\includegraphics[width=0.8\linewidth]{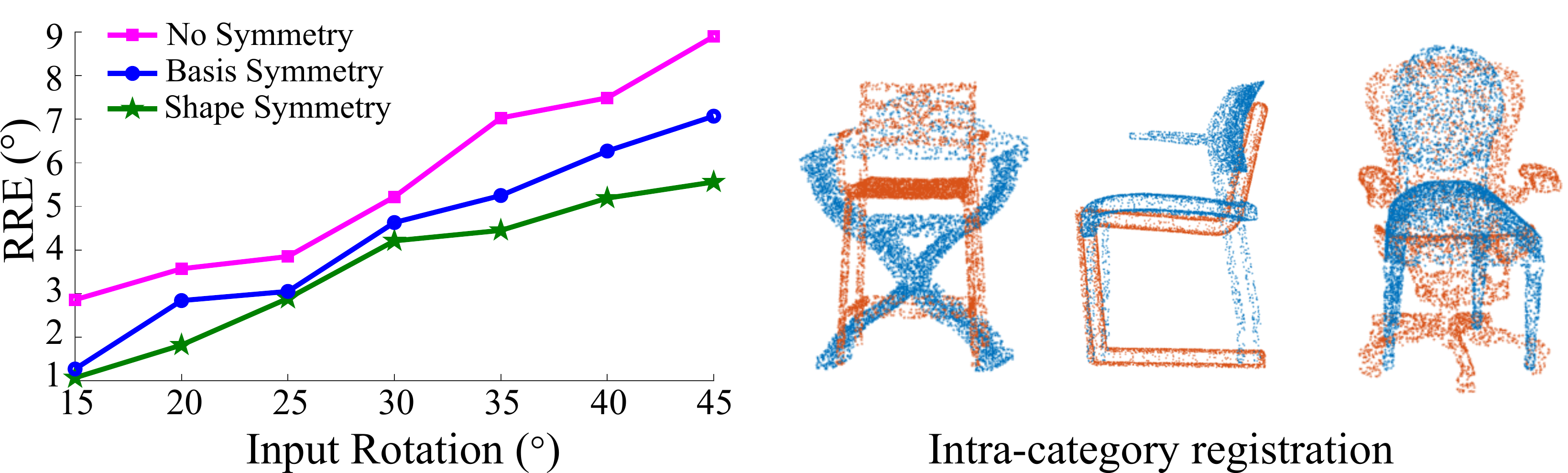}}
\caption{Left: Relative rotation error for different symmetry modelings. Right: 3 examples of registration between different instances of the same category.}
\label{fig:rot}
\end{figure}

\section{Conclusions} 
\label{sec:conclusion}

This paper investigates automatic discovery of kepoints in 3D misaligned point clouds that are consistent over inter-subject shape variations and intra-subject deformations in a category. We find that this can be solved, with unsupervised learning, by modeling keypoints with non-rigidity, based on symmetric linear basis shapes. Additionally, the proposed category-specific keypoints have one-to-one ordered correspondences and semantic consistency. Applications for the learned keypoints include registration, recognition, generation, shape completion and many more. Our experiments showed that high quality keypoints can be obtained using the proposed methods and that the method can be extended to complex non-rigid deformations. Future work could focus on better modeling complex deformations with non-linear approaches.

\subsection*{Acknowledgements}
This research was funded by the EU Horizon 2020 research and innovation program under grant agreement No.\ 820434. This work was also supported by Project RTI2018-096903-B-I00 (AEI/FEDER, UE) and Regional Council of Bourgogne Franche-Comté (2017-9201AAO048S01342).

%
{\small
\bibliographystyle{splncs}
\bibliography{egbib}

\begin{thebibliography}{10}

\bibitem{Lowe2004}
Lowe, D.G.:
\newblock Distinctive image features from scale-invariant keypoints.
\newblock International Journal of Computer Vision \textbf{60}(2) (2004)
  91--110

\bibitem{tola2009daisy}
Tola, E., Lepetit, V., Fua, P.:
\newblock Daisy: An efficient dense descriptor applied to wide-baseline stereo.
\newblock IEEE transactions on pattern analysis and machine intelligence
  \textbf{32}(5) (2009)  815--830

\bibitem{Bay2008}
Bay, H., Ess, A., Tuytelaars, T., {Van Gool}, L.:
\newblock Speeded-up robust features (surf).
\newblock Computer Vision and Image Understanding \textbf{110}(3) (2008)  346
  -- 359

\bibitem{novotny2019c3dpo}
Novotny, D., Ravi, N., Graham, B., Neverova, N., Vedaldi, A.:
\newblock C3dpo: Canonical 3d pose networks for non-rigid structure from
  motion.
\newblock In: Proceedings of the IEEE International Conference on Computer
  Vision. (2019)  7688--7697

\bibitem{Dai2012}
Dai, Y., Li, H., He, M.:
\newblock A simple prior-free method for non-rigid structure-from-motion
  factorization.
\newblock In: CVPR. (2012)

\bibitem{Snavely2007}
Snavely, N., Seitz, S.M., Szeliski, R.:
\newblock Modeling the world from internet photo collections.
\newblock Int. J. Comput. Vision \textbf{80}(2) (2007)  189--210

\bibitem{yew20183dfeat}
Yew, Z.J., Lee, G.H.:
\newblock 3dfeat-net: Weakly supervised local 3d features for point cloud
  registration.
\newblock In: European Conference on Computer Vision, Springer (2018)  630--646

\bibitem{Kneip2014}
Kneip, L., Li, H., Seo, Y.:
\newblock Upnp: An optimal o(n) solution to the absolute pose problem with
  universal applicability.
\newblock In Fleet, D., Pajdla, T., Schiele, B., Tuytelaars, T., eds.: Computer
  Vision -- ECCV 2014: 13th European Conference, Zurich, Switzerland, September
  6-12, 2014, Proceedings, Part I. (2014)

\bibitem{Luong1995}
Luong, Q.T., Faugeras, O.:
\newblock The fundamental matrix: theory, algorithms, and stability analysis.
\newblock International Journal of Computer Vision \textbf{17} (1995)  43--75

\bibitem{Loper2015}
Loper, M., Mahmood, N., Romero, J., Pons-Moll, G., Black, M.J.:
\newblock {SMPL}: A skinned multi-person linear model.
\newblock ACM Trans. Graphics (Proc. SIGGRAPH Asia) \textbf{34}(6) (2015)
  248:1--248:16

\bibitem{shotton2011real}
Shotton, J., Fitzgibbon, A., Cook, M., Sharp, T., Finocchio, M., Moore, R.,
  Kipman, A., Blake, A.:
\newblock Real-time human pose recognition in parts from single depth images.
\newblock In: CVPR 2011, Ieee (2011)  1297--1304

\bibitem{moreno20173d}
Moreno-Noguer, F.:
\newblock 3d human pose estimation from a single image via distance matrix
  regression.
\newblock In: Proceedings of the IEEE Conference on Computer Vision and Pattern
  Recognition. (2017)  2823--2832

\bibitem{cao2017}
Cao, Z., Simon, T., Wei, S.E., Sheikh, Y.:
\newblock Realtime multi-person 2d pose estimation using part affinity fields.
\newblock In: Proceedings of the IEEE Conference on Computer Vision and Pattern
  Recognition. (2017)  7291--7299

\bibitem{bogo2016smpl}
Bogo, F., Kanazawa, A., Lassner, C., Gehler, P., Romero, J., Black, M.J.:
\newblock Keep it smpl: Automatic estimation of 3d human pose and shape from a
  single image.
\newblock In: European Conference on Computer Vision, Springer (2016)  561--578

\bibitem{he2017mask}
He, K., Gkioxari, G., Doll{\'a}r, P., Girshick, R.:
\newblock Mask r-cnn.
\newblock In: Proceedings of the IEEE international conference on computer
  vision. (2017)  2961--2969

\bibitem{sattler2011fast}
Sattler, T., Leibe, B., Kobbelt, L.:
\newblock Fast image-based localization using direct 2d-to-3d matching.
\newblock In: 2011 International Conference on Computer Vision, IEEE (2011)
  667--674

\bibitem{tang2019cycle}
Tang, H., Xu, D., Liu, G., Wang, W., Sebe, N., Yan, Y.:
\newblock Cycle in cycle generative adversarial networks for keypoint-guided
  image generation.
\newblock In: Proceedings of the 27th ACM International Conference on
  Multimedia. (2019)  2052--2060

\bibitem{zafeiriou20173d}
Zafeiriou, S., Chrysos, G.G., Roussos, A., Ververas, E., Deng, J., Trigeorgis,
  G.:
\newblock The 3d menpo facial landmark tracking challenge.
\newblock In: Proceedings of the IEEE International Conference on Computer
  Vision Workshops. (2017)  2503--2511

\bibitem{huang2017coarse}
Huang, S., Gong, M., Tao, D.:
\newblock A coarse-fine network for keypoint localization.
\newblock In: Proceedings of the IEEE International Conference on Computer
  Vision. (2017)  3028--3037

\bibitem{pavlakos20176}
Pavlakos, G., Zhou, X., Chan, A., Derpanis, K.G., Daniilidis, K.:
\newblock 6-dof object pose from semantic keypoints.
\newblock In: ICRA. (2017)

\bibitem{zhang2014facial}
Zhang, Z., Luo, P., Loy, C.C., Tang, X.:
\newblock Facial landmark detection by deep multi-task learning.
\newblock In: European conference on computer vision, Springer (2014)  94--108

\bibitem{dong2018style}
Dong, X., Yan, Y., Ouyang, W., Yang, Y.:
\newblock Style aggregated network for facial landmark detection.
\newblock In: Proceedings of the IEEE Conference on Computer Vision and Pattern
  Recognition. (2018)  379--388

\bibitem{yu2016deep}
Yu, X., Zhou, F., Chandraker, M.:
\newblock Deep deformation network for object landmark localization.
\newblock In: European Conference on Computer Vision, Springer (2016)  52--70

\bibitem{alahi2012freak}
Alahi, A., Ortiz, R., Vandergheynst, P.:
\newblock Freak: Fast retina keypoint.
\newblock In: 2012 IEEE Conference on Computer Vision and Pattern Recognition,
  Ieee (2012)  510--517

\bibitem{li2019novel}
Li, Y.:
\newblock A novel fast retina keypoint extraction algorithm for multispectral
  images using geometric algebra.
\newblock IEEE Access \textbf{7} (2019)  167895--167903

\bibitem{li2019usip}
Li, J., Lee, G.H.:
\newblock Usip: Unsupervised stable interest point detection from 3d point
  clouds.
\newblock In: Proceedings of the IEEE International Conference on Computer
  Vision. (2019)  361--370

\bibitem{suwajanakorn2018discovery}
Suwajanakorn, S., Snavely, N., Tompson, J.J., Norouzi, M.:
\newblock Discovery of latent 3d keypoints via end-to-end geometric reasoning.
\newblock In: Advances in Neural Information Processing Systems. (2018)
  2059--2070

\bibitem{wu2016single}
Wu, J., Xue, T., Lim, J.J., Tian, Y., Tenenbaum, J.B., Torralba, A., Freeman,
  W.T.:
\newblock Single image 3d interpreter network.
\newblock In: European Conference on Computer Vision, Springer (2016)  365--382

\bibitem{yang2019perfect}
Yang, H., Carlone, L.:
\newblock In perfect shape: Certifiably optimal 3d shape reconstruction from 2d
  landmarks.
\newblock arXiv preprint arXiv:1911.11924 (2019)

\bibitem{hejrati2012analyzing}
Hejrati, M., Ramanan, D.:
\newblock Analyzing 3d objects in cluttered images.
\newblock In: Advances in Neural Information Processing Systems. (2012)
  593--601

\bibitem{wang2014robust}
Wang, C., Wang, Y., Lin, Z., Yuille, A.L., Gao, W.:
\newblock Robust estimation of 3d human poses from a single image.
\newblock In: Proceedings of the IEEE Conference on Computer Vision and Pattern
  Recognition. (2014)  2361--2368

\bibitem{persad2017automatic}
Persad, R.A., Armenakis, C.:
\newblock Automatic 3d surface co-registration using keypoint matching.
\newblock Photogrammetric engineering \& remote sensing \textbf{83}(2) (2017)
  137--151

\bibitem{mitra2014structure}
Mitra, N.J., Wand, M., Zhang, H., Cohen-Or, D., Kim, V., Huang, Q.X.:
\newblock Structure-aware shape processing.
\newblock In: ACM SIGGRAPH 2014 Courses.
\newblock (2014)  1--21

\bibitem{reed2013modeling}
Reed, M.P.:
\newblock Modeling body shape from surface landmark configurations.
\newblock In: International Conference on Digital Human Modeling and
  Applications in Health, Safety, Ergonomics and Risk Management, Springer
  (2013)  376--383

\bibitem{creusot20123d}
Creusot, C., Pears, N., Austin, J.:
\newblock 3d landmark model discovery from a registered set of organic shapes.
\newblock In: 2012 IEEE Computer Society Conference on Computer Vision and
  Pattern Recognition Workshops, IEEE (2012)  57--64

\bibitem{sridhar2019multiview}
Sridhar, S., Rempe, D., Valentin, J., Sofien, B., Guibas, L.J.:
\newblock Multiview aggregation for learning category-specific shape
  reconstruction.
\newblock In: Advances in Neural Information Processing Systems. (2019)
  2348--2359

\bibitem{gao2016symmetric}
Gao, Y., Yuille, A.L.:
\newblock Symmetric non-rigid structure from motion for category-specific
  object structure estimation.
\newblock In: European Conference on Computer Vision, Springer (2016)  408--424

\bibitem{Bregler2000}
Bregler, C., Hertzmann, A., Biermann, H.:
\newblock Recovering non-rigid 3{D} shape from image streams.
\newblock In: CVPR. (2000)

\bibitem{Torresani2008}
Torresani, L., Hertzmann, A., Bregler, C.:
\newblock Nonrigid structure-from-motion: Estimating shape and motion with
  hierarchical priors.
\newblock IEEE Trans. Pattern Anal. Mach. Intell. \textbf{30}(5) (2008)
  878--892

\bibitem{Akhter2008}
Akhter, I., Sheikh, Y., Khan, S., Kanade, T.:
\newblock Nonrigid structure from motion in trajectory space.
\newblock In: NIPS. (2008)

\bibitem{Taylor2010}
Taylor, J., Jepson, A.D., Kutulakos, K.N.:
\newblock Non-rigid structure from locally-rigid motion.
\newblock In: CVPR. (2010)

\bibitem{Parashar2016}
Parashar, S., Pizarro, D., Bartoli, A.:
\newblock Isometric non-rigid shape-from-motion in linear time.
\newblock In: CVPR. (2016)

\bibitem{kong2019deep}
Kong, C., Lucey, S.:
\newblock Deep non-rigid structure from motion.
\newblock In: Proceedings of the IEEE International Conference on Computer
  Vision. (2019)  1558--1567

\bibitem{wang2019normalized}
Wang, H., Sridhar, S., Huang, J., Valentin, J., Song, S., Guibas, L.J.:
\newblock Normalized object coordinate space for category-level 6d object pose
  and size estimation.
\newblock In: Proceedings of the IEEE Conference on Computer Vision and Pattern
  Recognition. (2019)  2642--2651

\bibitem{wu2020unsupervised}
Wu, S., Rupprecht, C., Vedaldi, A.:
\newblock Unsupervised learning of probably symmetric deformable 3d objects
  from images in the wild.
\newblock In: CVPR. (2020)

\bibitem{qi2017pointnet}
Qi, C.R., Su, H., Mo, K., Guibas, L.J.:
\newblock Pointnet: Deep learning on point sets for 3d classification and
  segmentation.
\newblock In: CVPR. (2017)

\bibitem{qi2017pointnet++}
Qi, C.R., Yi, L., Su, H., Guibas, L.J.:
\newblock Pointnet++: Deep hierarchical feature learning on point sets in a
  metric space.
\newblock In: Advances in neural information processing systems. (2017)

\bibitem{verma2018feastnet}
Verma, N., Boyer, E., Verbeek, J.:
\newblock Feastnet: Feature-steered graph convolutions for 3d shape analysis.
\newblock In: CVPR. (2018)

\bibitem{bromley1994signature}
Bromley, J., Guyon, I., LeCun, Y., S{\"a}ckinger, E., Shah, R.:
\newblock Signature verification using a" siamese" time delay neural network.
\newblock In: Advances in neural information processing systems. (1994)
  737--744

\bibitem{tomasi1992shape}
Tomasi, C., Kanade, T.:
\newblock Shape and motion from image streams under orthography: a
  factorization method.
\newblock International journal of computer vision \textbf{9}(2) (1992)
  137--154

\bibitem{yi2016scalable}
Yi, L., Kim, V.G., Ceylan, D., Shen, I.C., Yan, M., Su, H., Lu, C., Huang, Q.,
  Sheffer, A., Guibas, L.:
\newblock A scalable active framework for region annotation in 3d shape
  collections.
\newblock ACM Transactions on Graphics (TOG) \textbf{35}(6) (2016)  1--12

\bibitem{wu20153d}
Wu, Z., Song, S., Khosla, A., Yu, F., Zhang, L., Tang, X., Xiao, J.:
\newblock 3d shapenets: A deep representation for volumetric shapes.
\newblock In: CVPR. (2015)  1912--1920

\bibitem{li2018so}
Li, J., Chen, B.M., Hee~Lee, G.:
\newblock So-net: Self-organizing network for point cloud analysis.
\newblock In: Proceedings of the IEEE conference on computer vision and pattern
  recognition. (2018)  9397--9406

\bibitem{fan2017point}
Fan, H., Su, H., Guibas, L.J.:
\newblock A point set generation network for 3d object reconstruction from a
  single image.
\newblock In: Proceedings of the IEEE conference on computer vision and pattern
  recognition. (2017)  605--613

\bibitem{besl1992method}
Besl, P.J., McKay, N.D.:
\newblock Method for registration of 3-d shapes.
\newblock In: Sensor fusion IV: control paradigms and data structures. Volume
  1611., International Society for Optics and Photonics (1992)  586--606

\bibitem{bogo2017dynamic}
Bogo, F., Romero, J., Pons-Moll, G., Black, M.J.:
\newblock Dynamic faust: Registering human bodies in motion.
\newblock In: CVPR. (2017)  6233--6242

\bibitem{gerig2018morphable}
Gerig, T., Morel-Forster, A., Blumer, C., Egger, B., Luthi, M., Sch{\"o}nborn,
  S., Vetter, T.:
\newblock Morphable face models-an open framework.
\newblock In: 2018 13th IEEE International Conference on Automatic Face \&
  Gesture Recognition (FG 2018), IEEE (2018)  75--82

\bibitem{NEURIPS2019_9015}
Paszke, A., Gross, S., Massa, F., Lerer, A., Bradbury, J., Chanan, G., Killeen,
  T., Lin, Z., Gimelshein, N., Antiga, L., Desmaison, A., Kopf, A., Yang, E.,
  DeVito, Z., Raison, M., Tejani, A., Chilamkurthy, S., Steiner, B., Fang, L.,
  Bai, J., Chintala, S.:
\newblock Pytorch: An imperative style, high-performance deep learning library.
\newblock In: NIPS.
\newblock (2019)

\bibitem{kingma2014adam}
Kingma, D.P., Ba, J.:
\newblock Adam: A method for stochastic optimization.
\newblock arXiv preprint arXiv:1412.6980 (2014)

\bibitem{song2015sun}
Song, S., Lichtenberg, S.P., Xiao, J.:
\newblock Sun rgb-d: A rgb-d scene understanding benchmark suite.
\newblock In: Proceedings of the IEEE conference on computer vision and pattern
  recognition. (2015)  567--576

\end{thebibliography}
}
%

\pagebreak

\begin{center}
\textbf{\large Supplementary Material}

\end{center}

\setcounter{equation}{0}
\setcounter{figure}{0}
\setcounter{table}{0}
\setcounter{section}{0}
\setcounter{page}{1}
\makeatletter
\renewcommand{\thepage}{S\arabic{page}} 
\renewcommand{\theequation}{S\arabic{equation}}
\renewcommand{\thefigure}{S\arabic{figure}}
\renewcommand{\thesection}{S\arabic{section}}  
\renewcommand{\thetable}{S\arabic{table}}  

\section{Symmetry}

\subsection{Symmetric deformation space.}

\begin{proof}
The two linear spaces due to the two basis $\mathcal{B}_{\mathcal{C}\frac{1}{2}}$ and $\mathcal{B}'_{\mathcal{C}\frac{1}{2}}$ are symmetric by Definition 1 as $\mathcal{B}_{\mathcal{C}\frac{1}{2}}$ is symmetric to $\mathcal{B}'_{\mathcal{C}\frac{1}{2}}$ for any $K\in\mathbb{Z}$. Let $\mathsf{c}_i\in \mathcal{L}$ and $\mathsf{c}'_j\in\mathcal{L}'$ represent the respective half coefficients for any two shape instances $i$ and $j$, where $\mathcal{L}$ and $\mathcal{L}'$ defines the spaces of the predicted half coefficient vectors. Consequently, the actual deformation spaces are symmetric to one another if $\mathcal{L}$ and $\mathcal{L}'$ are equal. We define $p: p(\mathsf{c}_i)$ as the probability distribution of $\mathsf{c}_i$ and $q: q(\mathsf{c}'_j)$ as the probability distribution of $\mathsf{c}'_j$. If $p$ and $q$ come from the same distribution, we approach $p=q$. Then we have:

\begin{align}
\label{eq:symdefcond}
    \begin{split}
        &\text{if}\ \mathsf{c}_i=\mathsf{c}'_j, \\
        &\text{either},\ p(\mathsf{c}_i) = q(\mathsf{c}'_j) = 0, \\
        &\text{or},\ p(\mathsf{c}_i) >0\ \text{and} \ q(\mathsf{c}'_j) >0 \\
        &\text{for all},\ \mathsf{c}_i\in \mathcal{L}, \mathsf{c}'_j \in \mathcal{L}'.
    \end{split}
\end{align}
Condition \eqref{eq:symdefcond} guarantees that $\mathcal{L}=\mathcal{L}'$ and thus we obtain a symmetric deformation space.
\qed
\end{proof}

Note that for condition \eqref{eq:symdefcond} to be true, we do not require the two distributions to be equal, however, it is sufficient and desirable to have so. Therefore, Proposition 1 in the main text highlights such sufficient and desirable case. It is particularly meaningful when we are learning to predict the coefficients through stochastic methods such as a neural network training.
In our network architecture indeed one can expect the distributions of these two vectors to be similar given the data exhibits such a symmetric deformation space, since the prediction branches of $\mathsf{c}_i$ and $\mathsf{c}'_i$ are very similar. Alternatively, one may also try to enforce the condition using a KL divergence loss.

\begin{figure}
\centering
\subfloat{\includegraphics[width=0.8\linewidth]{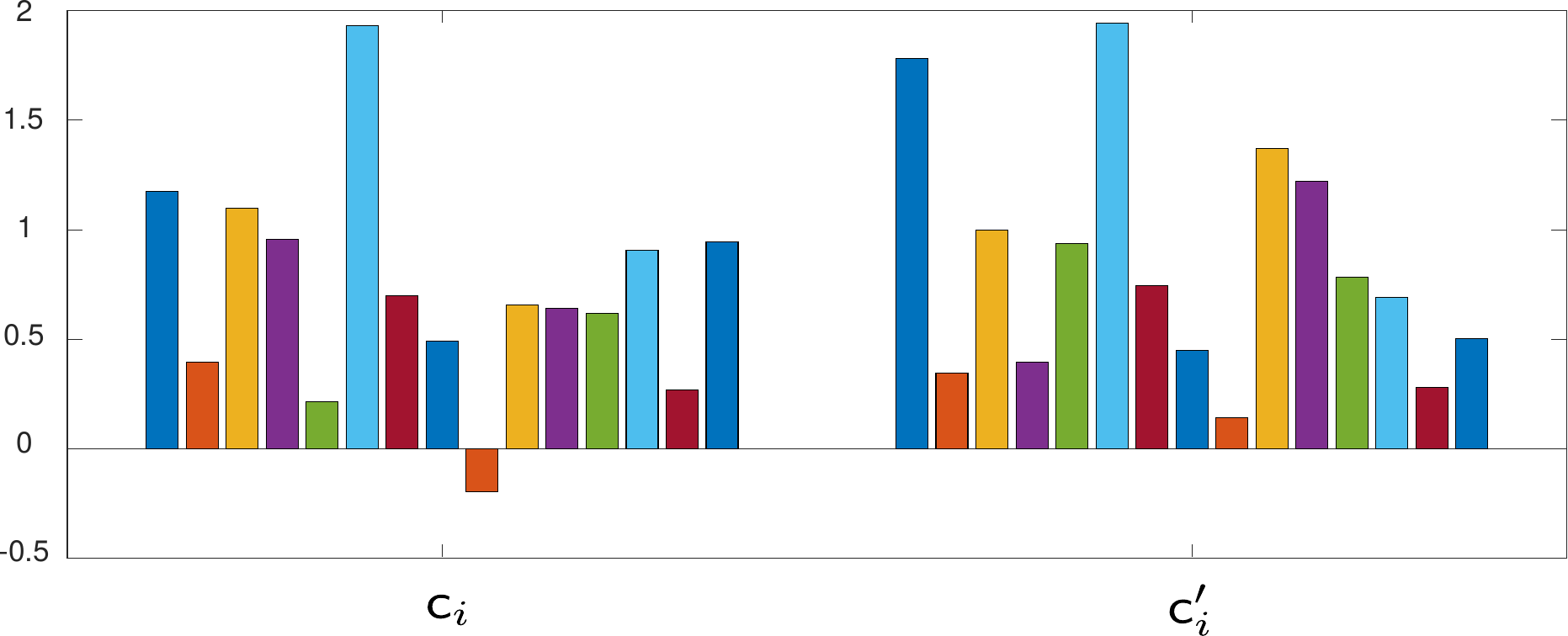}}
\caption{\textbf{Coefficients distribution}. Mean values of $\mathsf{c}_i$ components (left) and $\mathsf{c}'_i$ components (right) for the Dynamic FAUST~\cite{bogo2017dynamic}. The mean of the variances for the different components are: $\mathsf{c}_i: 0.54$, $\mathsf{c}'_i: 0.50$. The figure shows that the network learns similar distribution for the coefficients $\mathsf{c}_i$ and $\mathsf{c}'_i$.}
\label{fig:coefs}
\end{figure}

\subsection{Symmetry Plane Parametrization.} As mentioned in Sec. 5.3 in the main paper, we observe that handling misaligned data with unsupervised methods can lead to some rotation ambiguities. More specifically, we observe that different combination of basis shapes can result in different alignments.

As we show in Fig. 5 in the text, predicting the symmetry plane of the object category allows to have more control over the predicted instance poses. We came up with the idea of learning an additional common parameter, $\mathsf{R}_\mathcal{C}$, which is directly related to the symmetry plane. By adding this category-specific parameter, the network learns a common rotation for all the objects in the category. As a consequence, the instance-wise rotation, $\mathsf{R}_i$, can be thought like an offset from the reference basis alignment. 
Several evaluations confirmed that this strategy helps the learning process, reducing the rotation ambiguities.

\section{Additional Experiments}

\subsection{Keypoints correspondence}
We provide a complete overview for all the object categories evaluated regarding the keypoints correspondences across instances in Fig.~\ref{fig:kmeans}. This demonstrates the ability of our model to capture and model the inter-subject shape variations and intra-subject deformations in a category.

\begin{figure}
\centering
\subfloat{\includegraphics[width=1\linewidth]{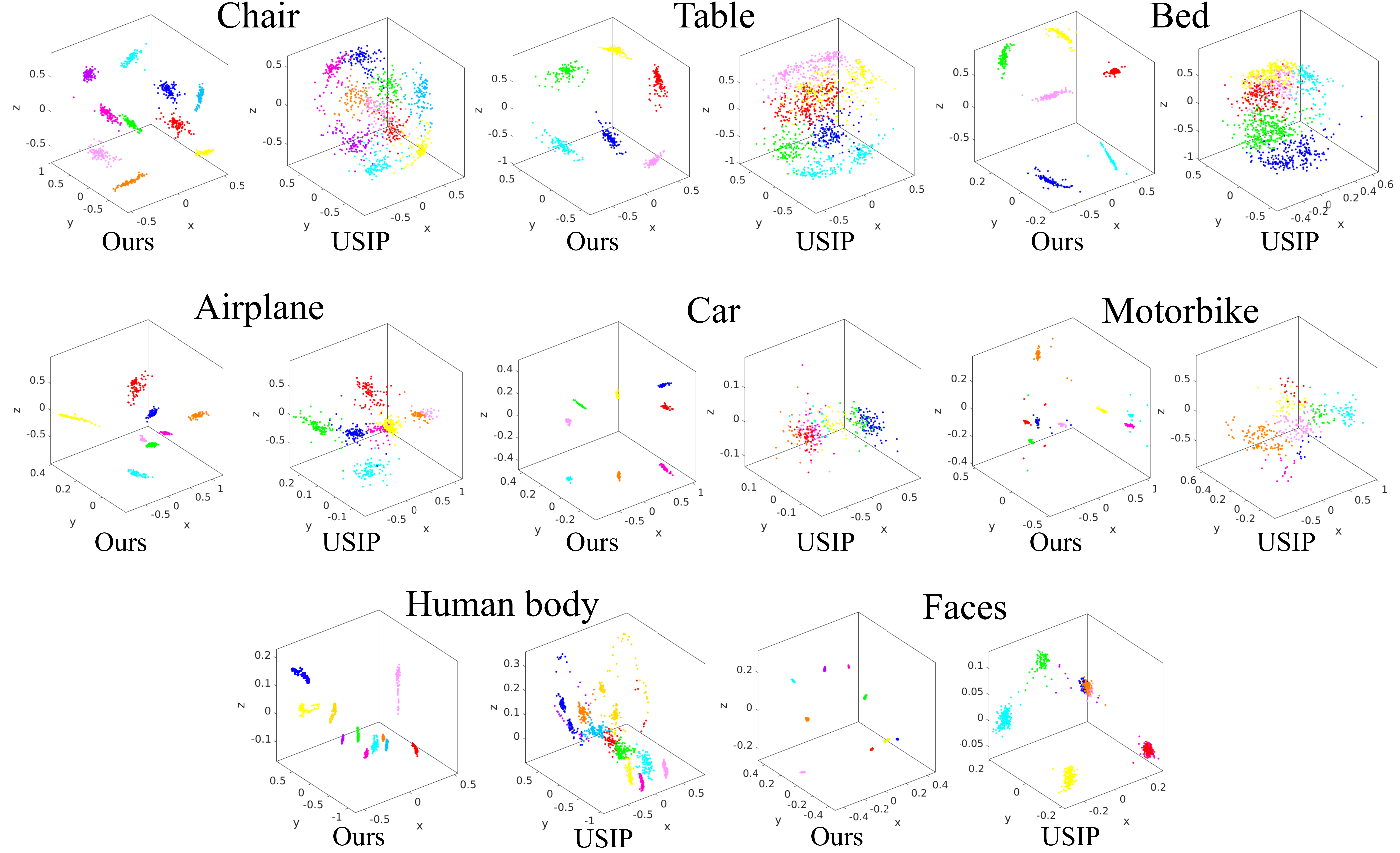}}
\caption{\textbf{Keypoints correspondence across instances}. We cluster the keypoints predicted for all the instances of a category to show their geometric consistency. Note how our keypoints get neatly clustered creating a general 3D shape template.}
\label{fig:kmeans}
\end{figure}

\subsection{Segmentation Label Transfer}
As demonstrated in Sec. 5.2 in the main paper, our predicted keypoints correspond to semantically meaningful locations. Therefore, here we explore the utility of the proposed category-specific keypoints for the segmentation label transfer task. 
In this experiment, for every point in the original shape $\mathsf{s}_{ij}\in\mathsf{S}_{i}$, we find its closest category-specific keypoint $\mathsf{p}_{ik}\in\mathsf{P}_{i}$, and transfer the corresponding semantic label to it. We assume the keypoints labels are known and correspond to those in Fig. 4 in the paper. 

Some qualitative results are shown in Fig.~\ref{fig:labeltr}. Our method achieves full correspondence between instances, therefore avoiding placing keypoints in less representative parts. An example is the engine, in grey, in the case of airplanes. This is reflected in the label transfer since there is no distinction of these parts. Besides that, only with eight keypoints in the example, we achieve reasonable results, close to the ground truth data.

\begin{figure}
\centering
\subfloat{\includegraphics[width=1\linewidth]{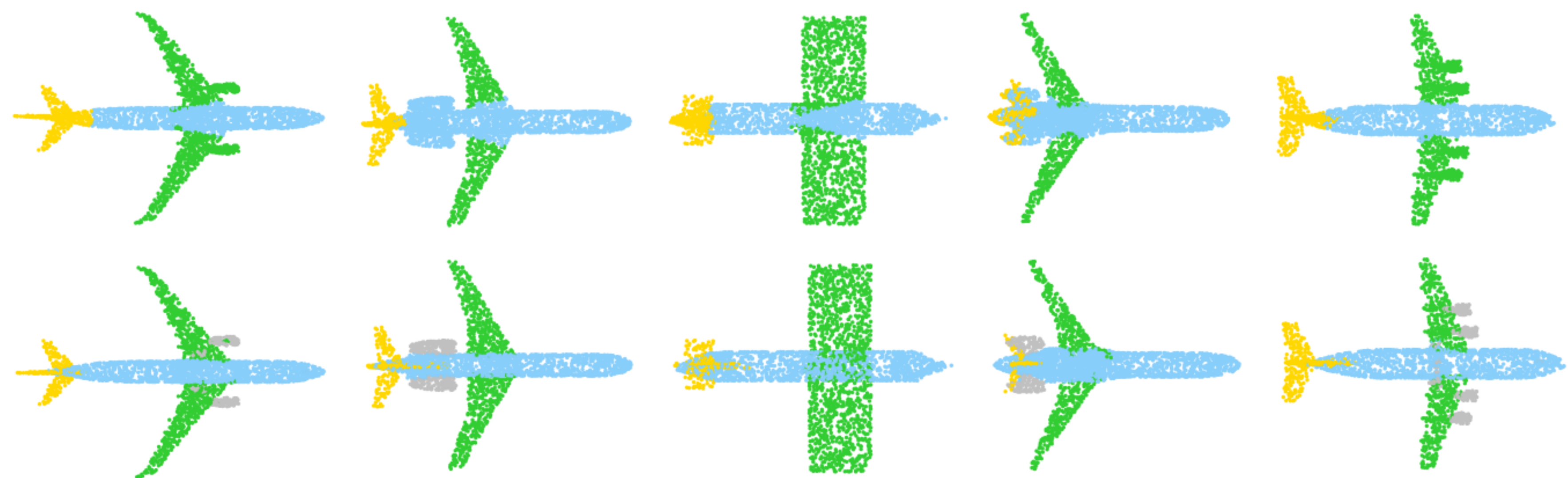}}
\caption{\textbf{First row:} results of performing semantic label transfer with our keypoints. \textbf{Second row:} ground truth. This is evaluated in ShapeNet part dataset \cite{yi2016scalable} using eight keypoints for the label transfer.}
\label{fig:labeltr}
\end{figure}

\subsection{Real Data}
In this section, we show the performance of our method for real data in Fig. \ref{fig:rd1}. For this experiment, the network is trained on the chair category from the ModelNet10 dataset \cite{wu20153d} and tested on real chairs from the SUNRGBD dataset \cite{song2015sun}. 
To generate the real data dataset from \cite{song2015sun}, we crop the points inside the ground truth 3D bounding boxes provided by the authors. 
Real data entail additional challenges. This is not only because shapes appear incomplete and noisy, but also because other objects may cause occlusions, e.g. part of a table occluding a chair. 
As illustrated in Fig. \ref{fig:rd1}, even though real data is fairly challenging, our network can still produce corresponding meaningful keypoints. 


Being able to generalize to previously unseen real objects as demonstrated in Fig. \ref{fig:rd1} is crucial and really useful for many tasks such as guide for shape completion or shape generation.

\begin{figure}
\centering
\subfloat{\includegraphics[width=1\linewidth]{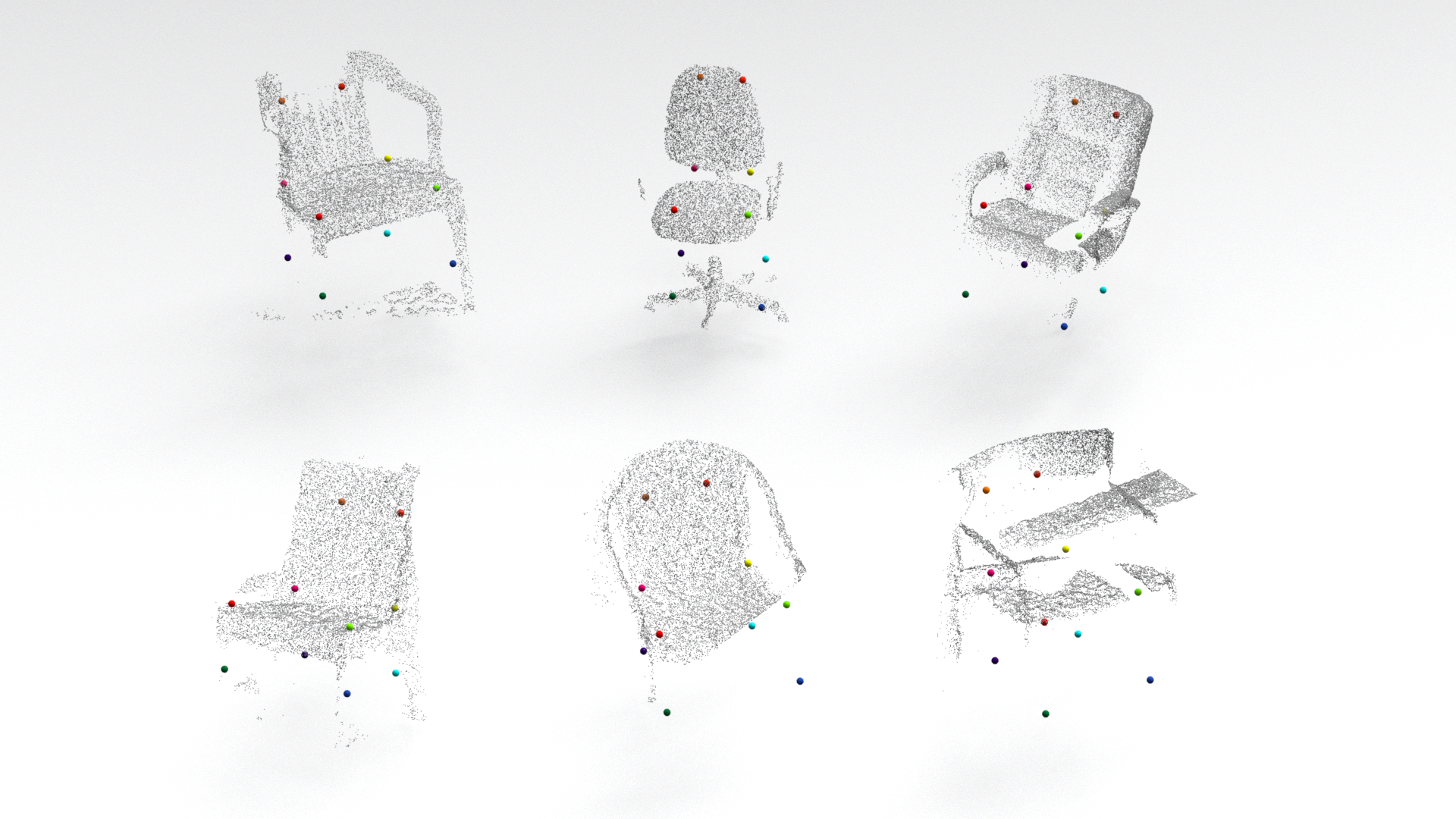}}
\caption{Results in real chairs from SUNRGBD dataset \cite{song2015sun} training with CAD chairs from ModelNet10 dataset \cite{wu20153d}.}
\label{fig:rd1}
\end{figure}


\section{Qualitative results}

In this section, we provide additional qualitative results on various object categories from the datasets evaluated in the paper; ModelNet10~\cite{wu20153d} in Fig. \ref{fig:quModelnet}, ShapeNet parts~\cite{yi2016scalable} in Fig. \ref{fig:quShapenet}, Dynamic FAUST~\cite{bogo2017dynamic} in Fig. \ref{fig:quDfaust} and Basel Face Model 2017~\cite{gerig2018morphable} in Fig. \ref{fig:quFaces}.

Again, we note that our network predicts corresponding keypoints between instances of the same category and consistently associates the same keypoint with
the same semantic part. For instance, for the chair object category, the keypoint colored in pink is always associated with the chair back, the keypoint colored in cyan is associated with the front left leg, etc. 

\begin{figure}
\centering
\subfloat{\includegraphics[width=1\linewidth]{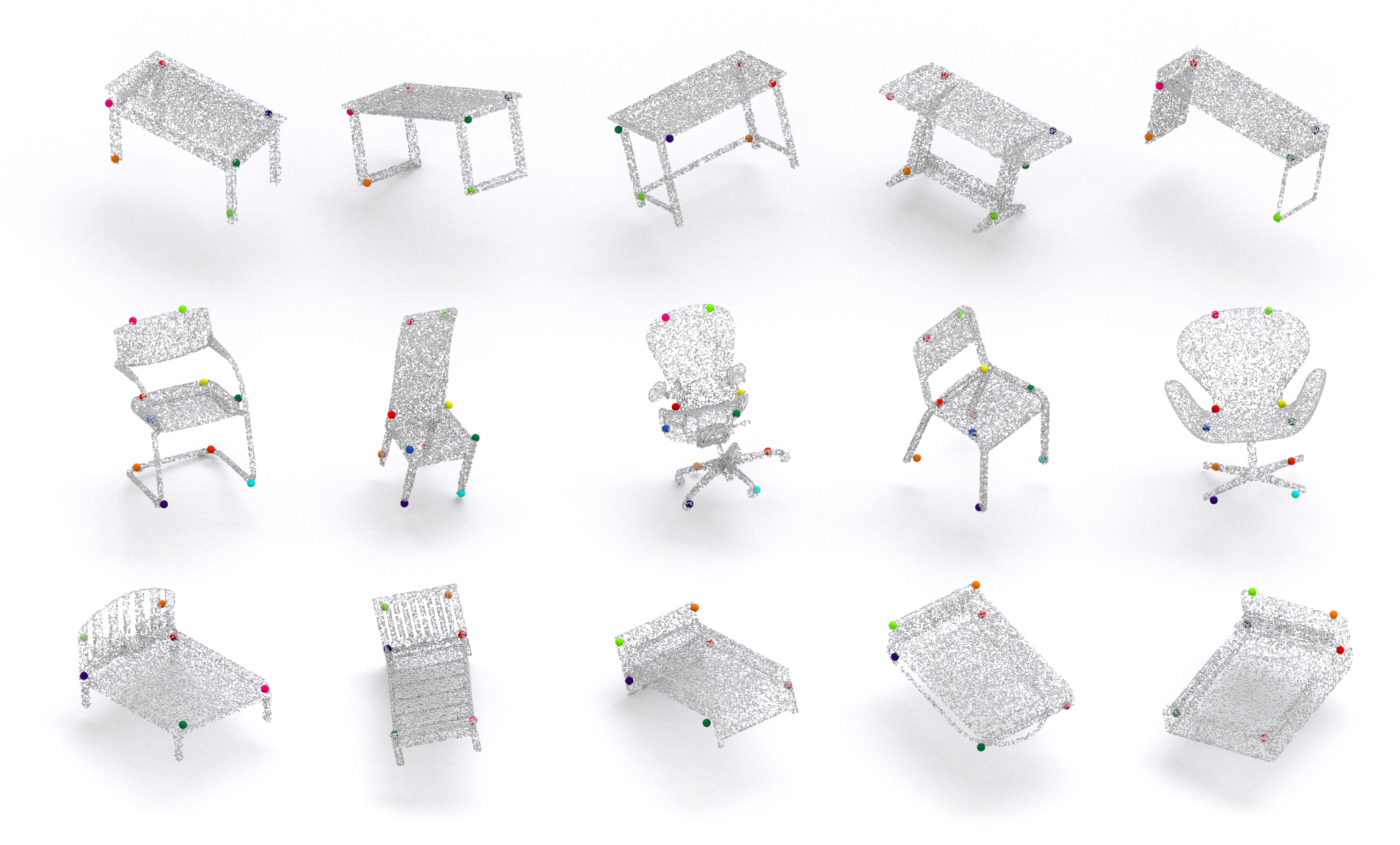}}
\caption{Qualitative results in table, chair and bed categories from ModelNet10 dataset \cite{wu20153d}.}
\label{fig:quModelnet}
\end{figure}

\begin{figure}
\centering
\subfloat{\includegraphics[width=1\linewidth]{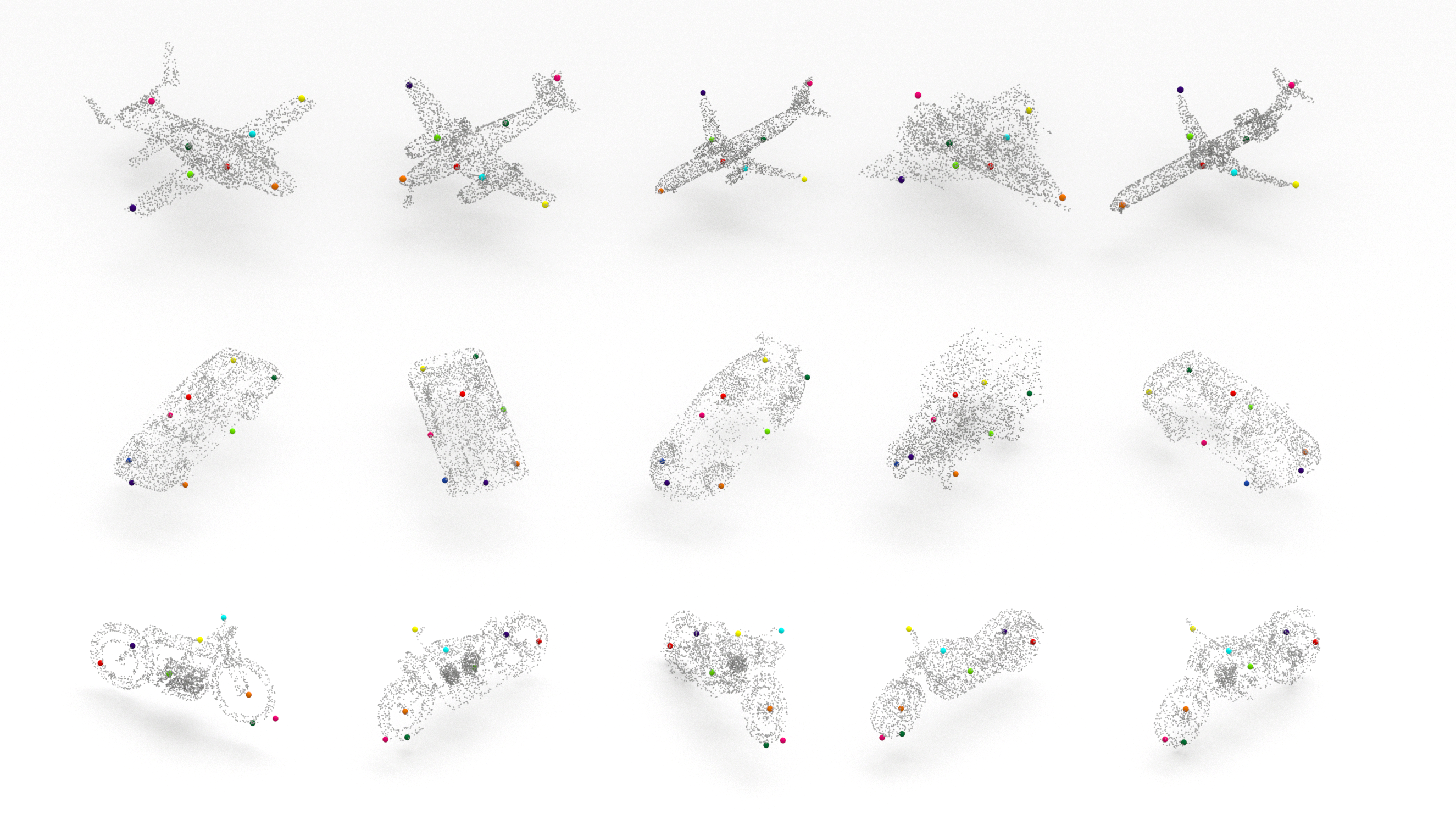}}
\caption{Qualitative results in airplane, car and motorbike categories from ShapeNet parts dataset \cite{yi2016scalable}.}
\label{fig:quShapenet}
\end{figure}

\begin{figure}
\centering
\subfloat{\includegraphics[width=1\linewidth]{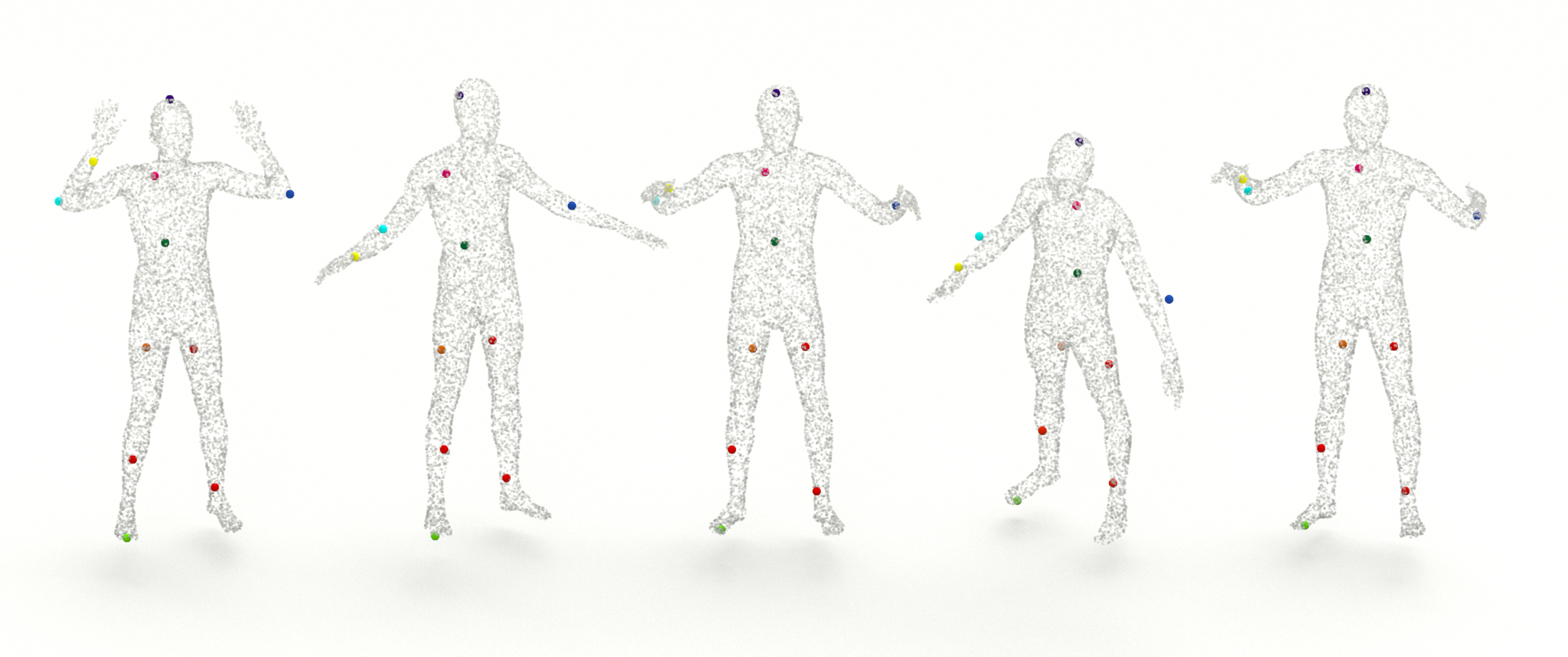}}
\caption{Qualitative results in human bodies from Dynamic FAUST dataset~\cite{bogo2017dynamic}.}
\label{fig:quDfaust}
\end{figure}

\begin{figure}
\centering
\subfloat{\includegraphics[width=1\linewidth]{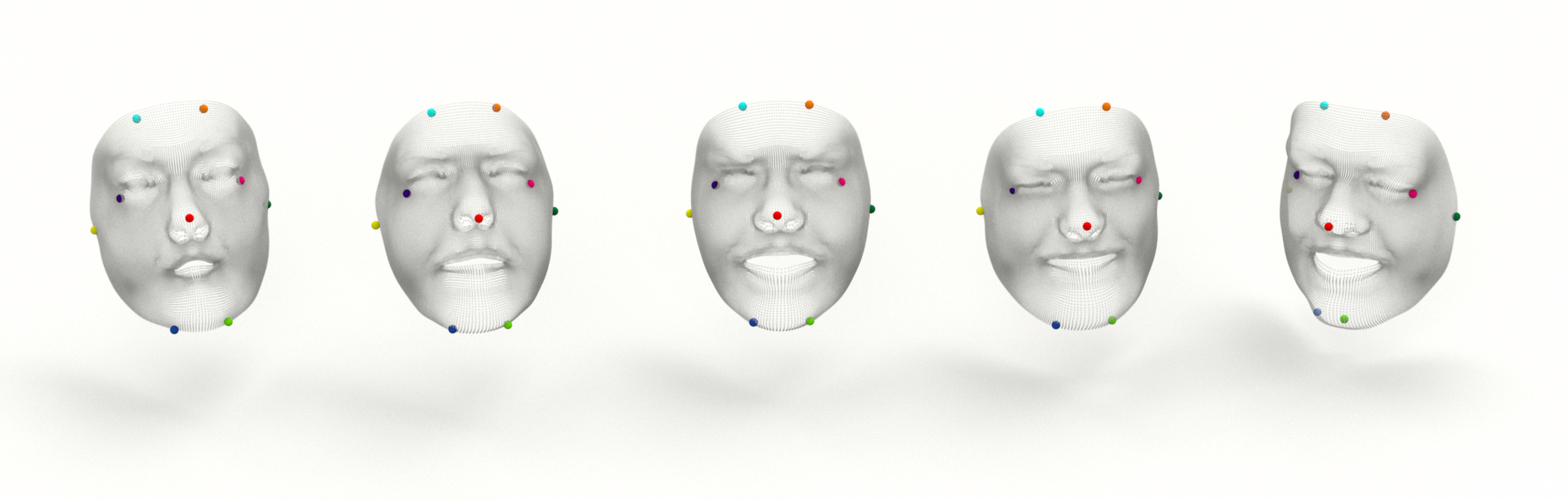}}
\caption{Qualitative results in faces from Basel Face Model 2017 dataset~\cite{gerig2018morphable}.}
\label{fig:quFaces}
\end{figure}

\end{document}